\documentclass[journal,twoside]{IEEEtran}
\pdfoutput=1

\usepackage{amsmath, amssymb, amsfonts, amsthm}
\usepackage{thmtools}
\usepackage{empheq}
\usepackage{siunitx}
\newtheorem{thm}{Theorem}[section]
\newtheorem{lem}[thm]{Lemma}
\newtheorem{prop}[thm]{Proposition}

\theoremstyle{definition}

\newtheorem{ass}{Assumption}

\usepackage[switch, pagewise]{lineno}
\usepackage[colorlinks]{hyperref}
\usepackage[noabbrev]{cleveref}

\crefname{equation}{eq.}{eqs.}
\crefname{figure}{fig.}{figs.}
\crefname{algorithm}{algorithm}{algorithms}
\crefname{table}{Table.}{Tables.}
\Crefname{table}{Table.}{Tables.}
\crefname{table*}{Table.}{Tables.}
\Crefname{table*}{Table.}{Tables.}

\crefname{lem}{Lemma.}{Lemmas.}
\Crefname{lem}{Lemma.}{Lemmas.}
\crefname{thm}{Theorem.}{Theorems.}
\Crefname{thm}{Theorem.}{Theorems.}
\crefname{prop}{Proposition.}{Propositions.}
\Crefname{prop}{Proposition.}{Propositions.}

\usepackage{algorithm}
\usepackage{algorithmic}

\usepackage{subfloat}
\usepackage{graphicx}
\usepackage[caption=false, font=footnotesize]{subfig}

\usepackage{filecontents}
\usepackage{pgfplots}
\usepackage{tikzscale}
\usepackage{tikz}
\usepgfplotslibrary{patchplots}

\usepackage{multirow}
\usepackage{makecell}

\DeclareMathOperator*{\argmin}{arg\,min\,}
\DeclareMathOperator*{\trace}{Tr}
\DeclareMathOperator*{\rank}{Rank}

\DeclareMathOperator*{\conv}{conv}
\DeclareMathOperator*{\sbjto}{s.t.}

\DeclareMathOperator*{\dprime}{\prime \prime}

\title{Online Structured Sparsity-based Moving Object Detection from Satellite Videos}

\author{
		Junpeng Zhang, ~\IEEEmembership{Student Member, ~IEEE, }
		Xiuping Jia, ~\IEEEmembership{Senior Member, ~IEEE,} 
		Jiankun Hu, ~\IEEEmembership{Senior Member, ~IEEE} and
		Jocelyn Chanussot, ~\IEEEmembership{Fellow, ~IEEE} 
}

\begin{document}
\pgfplotsset{compat=1.14}
	
\maketitle

\begin{abstract}
	
	Inspired by the recent developments in computer vision, low-rank and structured sparse matrix decomposition can be potentially be used for extract moving objects in satellite videos. 
	This set of approaches seeks for rank minimization on the background that typically requires batch-based optimization over a sequence of frames, which causes delays in processing and limits their applications.
	To remedy this delay, we propose an \textbf{O}nline \textbf{L}ow-rank and \textbf{S}tructured Sparse \textbf{D}ecomposition (O-LSD). 
	O-LSD reformulates the batch-based low-rank matrix decomposition with the structured sparse penalty to its equivalent frame-wise separable counterpart, which then defines a stochastic optimization problem for online subspace basis estimation. 
	In order to promote online processing, O-LSD conducts the foreground and background separation and the subspace basis update alternatingly for every frame in a video. 
	We also show the convergence of O-LSD theoretically.
	Experimental results on two satellite videos demonstrate the performance of O-LSD in term of accuracy and time consumption is comparable with the batch-based approaches with significantly reduced delay in processing.
	
\end{abstract}


\IEEEpeerreviewmaketitle

\section{Introduction}

\IEEEPARstart{O}{bject} detection on high resolution aerial images has been actively investigated in recent years \cite{han2016detection_survey_remote_sensing, xia2018dota}. 
Inspired by the state-of-the-art Deep Learning methods, such as Regional Convolution Neural Network (R-CNN) \cite{girshick2014rcnn}, Fast R-CNN \cite{girshick2015fast_rcnn}, Faster R-CNN \cite{ren2015faster_rcnn},  You Only Look Once (YOLO) \cite{redmon2016yolo} and Single Shot MultiBox Detector (SSD) \cite{liu2016ssd}, object detection performance on these images has been improved significantly \cite{long2017rcnn_rs_bbox, li2017faster_rcnn_rotation,ding2018faster_rcnn_rs,liu2018ssd_rs}.
These approaches are mainly exploring the spectral (or color) and spatial (texture or context) information on objects of interest, they detect objects of interest image by image, as none temporal information is available on those images. 
Recently, with satellite videos captured by Jilin-1 \cite{luo2017Jilin_1} and Skybox \cite{team2016planet}, dense temporal information becomes available, which benefits moving objects detection from space. 
Target tracking becomes possible and can then be conducted for various applications \cite{mou2016tracklet_satelite_video,du2018kcf_satellite_video,zhang2018bilevel_tracking,uzkent2018kcf_hsi}.

Detecting moving objects in a video is achieved by separating the temporal varying foreground, which is associated with the moving objects, and the background that lays in a low dimensional subspace from a video \cite{bouwmans2014rpca_video_review, bouwmans2017DLAM, bouwmans2018rpca_formulation_bg}. 
Given the moving objects account for a limited number of pixels in the foreground, the foreground is assumed sparse.
Robust Principle Component Analysis (RPCA), as one fundamental method in foreground extraction, defines a low-rank matrix decomposition problem with a sparse penalty \cite{candes2011RPCA}, which is solved by Principle Component Pursuit (RPCA-PCP) \cite{lin2011rpca_pcp,candes2011RPCA,wright2009RPCA_Proximal_Gradient} and Fast Low Rank Approximation (GoDec) \cite{zhou2011GoDec}. 
Based on the duality between sparsity and Laplace distribution, Probabilistic Robust Matrix Factorization (PRMF) provides a probabilistic interpretation to RPCA by combining Laplace error and Gaussian prior \cite{wang2012PRMF}.

As a moving object is commonly a set of neighboring pixels, spatial prior on the foreground is considered in low-rank matrix decomposition to improve moving object detection performance. 
Total Variation (TV) regularization is introduced to enforce smoothness on the foreground in the matrix decomposition \cite{xu2017rpca_vt_hsi}. 
DEtecting Contiguous Outliers in the LOw-rank Representation (DECOLOR) constrains the edges of moving objects to be contiguous, then first-order Markov Random Field (MRF) is integrated into low-rank matrix decomposition  \cite{zhou2013DECOLOR}.
Another possible spatial prior on the foreground is the sparsity over groups of spatial neighboring pixels on the foreground other than pixel-wise sparse, which is measured by Structured Sparsity-Inducing Norm  \cite{jenatton2011structured_sparsity}. 
Low-rank and Structured Sparse Decomposition (LSD) obeys this prior and penalizes the low-rank matrix decomposition by the structured sparsity-inducing norm of the foreground \cite{liu2015LSD}. 
As moving object detection in satellite video is more sensitive to random noises, integrating spatial prior should improve the quality of the estimated foreground, thus the moving object detection performance. 
By integrating structured sparsity, LSD presents boosted Moving Object Detection (MOD) performance in satellite videos \cite{zhang2019elsd}. 
DECOLOR, however, has a limited improvement in MOD performance for satellite videos, as the introduced MRF constraint tends to merge neighboring targets when the distance between them is too small.

\begin{table*}
	\caption{Comparison on Online Low-rank Matrix Decomposition Algorithms for Moving Object Detection}
	\label{tbl:comparison_methods}
	\centering
	\footnotesize
	\begin{tabular}{c|c|c|c|c|c}
		\Xhline{1pt}
		Method & Objective  Function& Constraints & Spatial Prior & \makecell{Optimization Scheme} & \makecell{Proven \\ Convergence}\\
		\Xhline{1pt}
		
		OPRMF \cite{wang2012PRMF} & \makecell{$\min \lambda \left\Vert \mathbf{D} - \mathbf{L} \mathbf{R} \right\Vert_{1}$ \\ $+ \frac{\lambda_{1}}{2} \left\Vert \mathbf{L} \right\Vert_{F}^{2} + \frac{\lambda_{2}}{2} \left\Vert \mathbf{R} \right\Vert_{F}^{2}$} & \makecell{ $\mathbf{L}_{ij} | \lambda_{1} \sim N(\mathbf{L}_{ij}| \mathbf{0}, \lambda_{1}^{-1} ),$ \\ $\mathbf{R}_{ij} | \lambda_{2} \sim N(\mathbf{R}_{ij}| \mathbf{0}, \lambda_{2}^{-1} )$ } & - & \makecell{Online Expectation \\ Maximization} & No \\
		\hline
		
		GRASTA \cite{he2012GRASTA} & $\min\left\Vert \mathbf{s}\right\Vert_{1}$  & \makecell{ $\mathbf{d} = \mathbf{L} \mathbf{r} + \mathbf{s}, $ \\ $\mathbf{L}^{T} \mathbf{L} = \mathbf{I}$}& -  & \makecell{Incremental Gradient \\  Descent Method on \\  Grassmannian Manifold} & No \\
		\hline
		
		OR-PCA \cite{feng2013ORPCA_SO} & 
		\makecell{$\min \frac{1}{2} \left\Vert \mathbf{D} - \mathbf{L} \mathbf{R} - \mathbf{S} \right\Vert^{2}_{F}  $ \\ $+ \frac{\lambda_{1}}{2} \left\Vert \mathbf{L} \right\Vert_{F}^{2}  + \frac{\lambda_{1}}{2} \left\Vert \mathbf{R} \right\Vert_{F}^{2} $ \\ $+ \lambda_{2} \left\Vert \mathbf{S}\right\Vert_{1}$ }& - & -  & \makecell{Stochastic \\ Optimization} & \textbf{Yes} \\
		\Xhline{1pt}
		
		COROLA \cite{shakeri2016COROLA} & \makecell{$\min \frac{1}{2} \left\Vert \mathcal{P}_{S^{\perp}}(\mathbf{D} - \mathbf{L} \mathbf{R} ) \right\Vert_{F}^{2}$ \\ $ + \frac{\alpha}{2} \left\Vert \mathcal{P}_{S^{\perp}}(\mathbf{L}) \right\Vert_{F}^{2}$ \\ $ + \frac{\alpha}{2} \left\Vert \mathcal{P}_{S^{\perp}}(\mathbf{R}) \right\Vert_{F}^{2} $ \\ $ + \beta \left\Vert \mathbf{S} \right\Vert_{1} + \gamma \left\Vert \mathbf \mathbf{A} vec(\mathbf{S}) \right\Vert $} & $\mathbf{S}_{ij} \in \{0,1\}$ & \makecell{Edge \\ Contiguousness} & \makecell{Stochastic \\ Optimization} & \textbf{Yes} \\ 
		\Xhline{1pt}
		
		GOSUS \cite{xu2013GOSUS} & \makecell{$\min \sum_{i=1}^{l} \mu_{i} \left\Vert \mathbf{G \mathbf{s}} \right\Vert_2$ \\ $+ \frac{\lambda}{2} \left\Vert \mathbf{L} \mathbf{r} + \mathbf{s} - \mathbf{d} \right\Vert_{2}^{2}$ }& $\mathbf{L}^{T} \mathbf{L} = \mathbf{I}$ & \makecell{\textbf{Structured Sparsity} \\ \textbf{of Foreground} }  &  \makecell{Incremental Gradient \\  Descent Method on \\  Grassmannian Manifold} & No \\
		\hline
		
		\textbf{Proposed O-LSD} & \makecell{$\min \frac{1}{2} \left\Vert \mathbf{D} - \mathbf{L} \mathbf{R} - \mathbf{S} \right\Vert^{2}_{F}$ \\ $ + \frac{\lambda_{1}}{2} \left\Vert \mathbf{L} \right\Vert_{F}^{2} + \frac{\lambda_{1}}{2} \left\Vert \mathbf{R}  \right\Vert_{F}^{2} $ \\ $ +  \lambda_{2} \sum_{\mathbf{s} \in \mathbf{S}} \left\Vert \mathbf{s} \right\Vert_{\ell_{1}/\ell_{\infty}} $ }& - & \makecell{\textbf{Structured Sparsity} \\ \textbf{of Foreground} }  & \makecell{Stochastic \\ Optimization} & \textbf{Yes} \\ 
		\Xhline{1pt}
		
		\multicolumn{6}{l}{\makecell{$*$There exist a variety of low-rank decomposition algorithm in the literature, however, we select the most related works here. Interested readers \\ may refer to \cite{bouwmans2014rpca_video_review,bouwmans2017DLAM,bouwmans2018rpca_formulation_bg,javed2019sssr} for more comprehensive reviews in low-rank matrix decomposition and their applications in video processing.}} \\
	\end{tabular} 
\end{table*}

Regardless of the detection performance of the algorithms above, a pitfall of them is that their solutions are based on optimization in a batch manner. 
These approaches use Singular Value Decomposition (SVD) for low-rank background estimation, which couples all the samples in each iteration of the optimization. 
The detection results are not available until the optimization terminates, which results in delays in processing.
Another shortcoming of batch-based approaches is the difficulty in handling a video with an incremental length. 
Both issues of the batch-based algorithms limit their application in various online systems.

In order to reduce the delay and to make MOD adaptable to videos of incremental length, online method is expected to sequentially estimate foreground and background for each new incoming frame. 
By low-rank matrix decomposition, the estimated Principal Component basis vectors represent a subspace, where the background lays. 
This subspace can be identified by a point on the Grassmannian manifold, and the incremental gradient descent method on Grassmannian Manifold is employed for online subspace tracking or updating \cite{turaga2008Stiefel_Grassmannian_manifold_cv, balzano2010GROUSE, harandi2013dictionary_learning_grassmannian_manifold}. 
For moving object detection, Grassmannian Robust Adaptive Subspace Tracking Algorithm (GRASTA) \cite{he2012GRASTA} and Grassmannian Online Subspace Updates with Structured-sparsity (GOSUS) \cite{xu2013GOSUS} are developed for online low-rank matrix decomposition with the pixel-wise sparse penalty and the structured sparse penalty, respectively. 
These set of approaches, however, provide no theoretical guarantee on their convergence, and their performance is heavily sensitive to the selection of the learning rate.

Another possibility for online low-rank matrix decomposition is the increasingly common matrix factorization approximation of nuclear norm, where rank-minimization is replaced by the sum of square penalties of its factorization \cite{recht2010guaranteed_min_rank,feng2013ORPCA_SO,sprechmann2015MF_PAMI_online,shakeri2016COROLA,shen2016online_subspace_clustering}. 
With this reformulation, iterative optimization scheme is then developed based on stochastic optimization for solving low-rank decomposition problem online. 
Online Robust Principle Component Analysis (OR-PCA) solves the online low-rank matrix decomposition problem with pixel-wise sparsity penalty, which, more importantly, proves that iterative optimization algorithm converges to the global optimum of the original RPCA approach \cite{feng2013ORPCA_SO}. 
Utilizing first-order Markov Random Field, spatial prior on contiguous edges is also integrated with this reformulation for online low-rank matrix decomposition \cite{shakeri2016COROLA}.

It has been observed that, by introducing the structured sparse penalty, LSD boosts the moving object detection performance in satellite videos \cite{zhang2019elsd}. 
While GOSUS combines online low-dimensional subspace tracking with structured sparsity, no theoretical guarantee on its convergence is provided. 
To the best of our knowledge, there exists a gap between online algorithm with theoretically guaranteed convergence and the one with the structured sparsity penalty, as demonstrated in \Cref{tbl:comparison_methods}. 
In order to fill this gap, we present an online low-rank matrix decomposition approach with structured sparse penalty, named as \textbf{O}nline \textbf{L}ow-rank and \textbf{S}tructured Sparse \textbf{D}ecomposition (O-LSD), which not only combines the structured sparsity penalty but also provides theoretically guaranteed convergence.
We follow the matrix factorization approximation of nuclear norm for online learning in \cite{feng2013ORPCA_SO,sprechmann2015MF_PAMI_online,shakeri2016COROLA,shen2016online_subspace_clustering}, 
and decompose the background matrix to a set of background frames that are reconstructed by the estimated subspace basis and their associated coefficients. 
To promote online processing, the proposed O-LSD algorithm is composed of two building blocks. 
For each frame, its corresponding foreground and background frames are reconstructed by the current subspace basis, then the subspace basis is updated for this new input. 
This procedure defines a stochastic optimization problem, and we show that O-LSD algorithm converges almost surely. 
Existing convergence analysis in \cite{mairal2010online_dictionary_learning,feng2013ORPCA_SO, shen2016online_subspace_clustering,dohmatob2016online_structured_dictionary_learning_st_lap} is built on necessary and sufficient conditions on the unique solution in sparse encoding. 
In O-LSD, no such conditions exist for structured sparsity encoding to the best of our knowledge, and we show the convergence of O-LSD based on the boundedness of the sub-gradients in structured sparsity encoding, then a set of related properties of O-LSD are demonstrated. 
Experimental evaluations and analysis were performed on a satellite video dataset with two videos, where we compared our algorithm with five state-of-the-art algorithms.

In summary, the main contributions of this work are four-fold: 
\begin{enumerate}
	\item We propose an \textbf{O}nline \textbf{L}ow-rank and \textbf{S}tructured Sparse \textbf{D}ecomposition (O-LSD) for moving object detection in satellite videos by reformulating batch-based LSD using the matrix factorization approximation of nuclear norm.
	
	\item To solve the new reformulated optimization problem, two iterating steps are designed and developed.
	For each frame, the corresponding foreground and background frames are first reconstructed by the current subspace basis, then the subspace basis is updated by the given frame.
	
	\item We show that O-LSD converges almost surely. 
	In contrary to most current online algorithms, we show that O-LSD can converge without meeting the conditions for the unique solution of structured sparsity encoding. 
	Due to the lack of these conditions, the solution of O-LSD can converge to neither a stationary point nor the global optimum of its batch-based counterpart LSD. 
	This finding and its corresponding proof are useful beyond the scope of this paper.
	
	\item Due to the better convergence characteristics of O-LSD, it can further reduce its processing delay with negligible effects on the detection performance, by down-sampling in the temporal domain. 
\end{enumerate}

The remainder of this paper is organized as follows. 
The proposed O-LSD is presented in \Cref{sec:method}, where its convergence analysis is provided in \Cref{subsec:convergence}. 
The experimental parameter settings and performance comparison against state-of-the-art approaches are presented in \Cref{sec:experiments}.
Finally, conclusions and suggestions for future research are given in Section \ref{sec:conclusion}.

\section{Proposed Method}
\label{sec:method}

\subsection{Matrix Factorization to LSD}

\textbf{L}ow-rank and \textbf{S}tructured Sparse \textbf{D}ecomposition (LSD) seeks for a low-rank matrix decomposition of an observation matrix, which at the same time imposes the structured sparse penalty on the foreground. 
Given a fixed-length sequence from $n$ video frames and each frame contains $p$ pixels, LSD decomposes its corresponding matrix $\mathbf{D} \in \mathbb{R}^{p \times n}$ to a low-rank background matrix $\mathbf{B} \in \mathbb{R}^{p \times n}$ plus a structured sparse foreground matrix $\mathbf{S} \in \mathbb{R}^{p \times n}$, and defines a batch-based optimization problems as
\begin{equation}
\label{eq:LSD}
\begin{aligned}
(\mathbf{B}^{*}, \mathbf{S}^{*}) =  & \argmin_{\mathbf{B}, \mathbf{S} }  {
	\left\Vert \mathbf{B} \right\Vert_{*} + 
	\lambda \Omega(\mathbf{S}) }\\ 
\sbjto & \mathbf{D} = \mathbf{B} + \mathbf{S} 
\end{aligned},
\end{equation}
where $\Omega(\mathbf{S})$ refers to the structured sparsity-inducing norm of $\mathbf{S}$ and $\lambda$ is a scalar that assigns the weight of structured sparsity.
The structured sparsity-inducing norm  \cite{jenatton2011structured_sparsity,jenatton2010sparse_hieriarchical_dictionary_learning,jia2012dictionary_learning_structured_sparsity} indicates the sparsity over groups of neighboring pixels as
\begin{equation}
\begin{aligned}
\Omega(\mathbf{S}) = \sum_{\mathbf{s} \in \mathbf{S}} \left\Vert \mathbf{s}\right\Vert_{\ell_{1}/\ell_{\infty}} 
= \sum_{\mathbf{s} \in \mathbf{S}} \sum_{g \in \mathcal{G}} \eta_{g} \left\Vert \mathbf{s}_{|g} \right\Vert_{\infty} \\
\end{aligned},
\end{equation}
where $\mathcal{G}$ defines the set of groups of neighboring pixels, and $\mathbf{s}_{|g} \in \mathbb{R}^{p}$ is a sparse vector with non-zero elements at the indices represented in a group $g \in \mathcal{G}$.
$\eta_{g}$ specifies the weight for a group of the pixels. 
In this paper, we assume each group contributes equally and assign 1.0 to it, $\eta_{g} = 1.0, \forall g \in \mathcal{G}$.
In LSD, no temporal prior or constraints on the foreground are considered, thus the structured sparse penalty over a sequence of frames is frame-wise separable. 
In this paper, the groups of spatially related pixels $\mathcal{G}$ is constructed by $3 \times 3$ grid scanning over the foreground, as the moving targets in satellite videos are usually in small scales.

Inspired by \cite{feng2013ORPCA_SO}, the equality constraint in \Cref{eq:LSD} is removed, and we obtain a reformulated optimization problem
\begin{equation}
\label{eq:Lagrangian_LSD}
\begin{aligned}
(\mathbf{B}^{*}, \mathbf{S}^{*}) =  \argmin_{\mathbf{B}, \mathbf{S} }  \lambda_{1} \left\Vert \mathbf{B} \right\Vert_{*} &
+ \lambda_{2} \sum_{\mathbf{s} \in \mathbf{S}} \left\Vert \mathbf{s}\right\Vert_{\ell_{1}/\ell_{\infty}} \\
& + \frac{1}{2}  \left\Vert \mathbf{D} - \mathbf{B} - \mathbf{S} \right\Vert^{2}_{F},
\end{aligned}
\end{equation}
in which $\lambda_1 > 0$ and $\lambda_2 > 0$ are the corresponding weights for the low-rank penalty and the structured sparsity penalty.

Guided by the trending reformulation by matrix factorization in \cite{feng2013ORPCA_SO,sprechmann2015MF_PAMI_online, javed2019sssr}, we replace the low-rank term $\left\Vert \mathbf{B} \right\Vert_{*}$ in \Cref{eq:Lagrangian_LSD} by its approximation, which makes use of the following lemma. 

\begin{lem}
	Given that $\mathbf{B}$ is factorized as $\mathbf{B} = \mathbf{L} \mathbf{R}$, $L \in \mathbb{R}^{p*r}, \mathbf{R} \in \mathbb{R}^{r*n}$, the nuclear norm of $\mathbf{B}$ is upper bounded by the sum of Frobenius norms of $\mathbf{L}$ and $\mathbf{R}$, as
	\begin{equation}
	\label{eq:nuclear_norm_factorization}
	\left\Vert \mathbf{B} \right\Vert_{*} = \inf_{\mathbf{L} \in \mathbb{R}^{p \times r}, \mathbf{R} \in \mathbb{R}^{r \times n}}
	\left\lbrace \frac{1}{2}\left\Vert \mathbf{L} \right\Vert_{F}^{2} +\frac{1}{2}  \left\Vert \mathbf{R} \right\Vert_{F}^{2}:
	\mathbf{B} = \mathbf{L} \mathbf{R}
	\right\rbrace 
	\end{equation}
	When $r > \rank(\mathbf{B})$, the jointly non-convex quadratic optimization problem \Cref{eq:nuclear_norm_factorization} is equivalent to minimize the nuclear norm of $\mathbf{B}$.
	\footnote[1]{Please refer to \cite{recht2010guaranteed_min_rank} for detailed proof.}
\end{lem}

By substituting $\left\Vert \mathbf{B} \right\Vert_{*}$ with its factorized approximation, we rewrite the optimization problem in \Cref{eq:Lagrangian_LSD} to
\begin{equation}
\label{eq:factorized_opitmization_problem}
\begin{aligned}
(\mathbf{L}^{*}, \mathbf{R}^{*}, \mathbf{S}^{*}) =  \argmin_{\mathbf{L}, \mathbf{R}, \mathbf{S}} 
& \frac{1}{2} \left\Vert \mathbf{D} - \mathbf{L} \mathbf{R} - \mathbf{S} \right\Vert^{2}_{F}
+ \frac{\lambda_{1}}{2} \left\Vert \mathbf{L} \right\Vert_{F}^{2} \\
& + \frac{\lambda_{1}}{2} \left\Vert \mathbf{R}  \right\Vert_{F}^{2} 
+  \lambda_{2} \sum_{\mathbf{s} \in \mathbf{S}} \left\Vert \mathbf{s} \right\Vert_{\ell_{1}/\ell_{\infty}} ,
\end{aligned}
\end{equation}
where $\mathbf{L} \in \mathbb{R}^{p*r}$ is considered as the subspace basis of the background matrix $\mathbf{B}$, and $\mathbf{R} \in \mathbb{R}^{r*n}$ is the coefficients to reconstruct $\mathbf{B}$ with given $\mathbf{L}$. 
$r$ is the estimated dimension of the subspace that the background frames lay in.

With a pair of estimated $\mathbf{L}$ and $\mathbf{R}$, each column vector in $\mathbf{R}$ corresponds to an estimated background frame in $\mathbf{B}$.
Let $\mathbf{d}_t$, $\mathbf{r}_{t}$ and $\mathbf{s}_t$ refer to the $t$-th column of $\mathbf{D}$, $\mathbf{R}$ and $\mathbf{S}$ respectively, this optimization problem in \Cref{eq:factorized_opitmization_problem} is equivalent to minimizing an empirical cost function
\begin{equation}
\label{eq:empirical_cost_function}
\begin{aligned}
f_{n}(\mathbf{L}) = \frac{1}{n} \sum_{i=1}^{n} \ell(\mathbf{D}_{i}, \mathbf{L}) 
+ \frac{\lambda_{1}}{2n} \left\Vert \mathbf{L} \right\Vert_{F}^2 
\end{aligned},
\end{equation}
in which $\ell(\mathbf{d}_{i}, \mathbf{L})$ is the reconstruction cost evaluated with fixed $\mathbf{L}$ by
\begin{equation}
\label{eq:reconstruction_cost_function}
\begin{aligned}
&\ell(\mathbf{d}, \mathbf{L}) = 
\min_{\mathbf{r}, \mathbf{s}} \hat{\ell}(\mathbf{d}, \mathbf{L}, \mathbf{r}, \mathbf{s}), \\
\hat{\ell}(\mathbf{d}, \mathbf{L}, \mathbf{r}, \mathbf{s})
& = \frac{1}{2} \left\Vert \mathbf{d}- \mathbf{L} \mathbf{r} - \mathbf{s} \right\Vert^{2}_{2}
+  \frac{\lambda_{1}}{2} \left\Vert \mathbf{r}  \right\Vert_{2}^{2} 
+  \lambda_{2} \left\Vert \mathbf{s} \right\Vert_{\ell_{1}/\ell_{\infty}}.
\end{aligned}
\end{equation}
Minimization of the empirical cost function in \Cref{eq:empirical_cost_function} associates the sum of reconstruction costs $\ell(\cdot, \mathbf{L})$, where, for each frame, $(\mathbf{r}, \mathbf{s})$ is optimized with the optimization target $\mathbf{L}$ as a parameter, whose formulation fits to the max-min optimization problem.

\subsection{Online LSD}

Through the above reformulation, it is still impossible to update $\mathbf{L}$ without re-estimating all pairs of $(\mathbf{r}, \mathbf{s})$, which obstructs processing in an online fashion.
In order to promote online processing, we propose an online algorithm, named \textbf{O}nline \textbf{L}ow-Rank and \textbf{S}tructured Sparse \textbf{D}ecomposition (O-LSD), where Foreground and Background Separation and Subspace Basis Update are sequentially conducted for each frame.

O-LSD is an online algorithm that processes an input frame at each time instance in an online manner. 
At each time instance $t$, we have obtained $\mathbf{L}_{t-1}$ estimated from previous time instance $t-1$. 
The foreground frame and background frame are separated by solving the following optimization problem 
\begin{equation}
\label{eq:solve_r_c}
\begin{aligned}
(\mathbf{r}_{t}^{*}, \mathbf{s}_{t}^{*}) = \argmin_{\mathbf{r}_{t}, \mathbf{s}_{t}}
\hat{\ell}(\mathbf{d}_{t}, \mathbf{L}_{t-1}, \mathbf{r}_{t}, \mathbf{s}_{t}).
\end{aligned}
\end{equation}
We term this procedure as Foreground and Background Separation, which is detailed \Cref{subsec:solve_r_s}.

Then subspace basis updating is performed with all pair of $\mathbf{r}_{i}$ and $\mathbf{s}_{i}$, $i \in \{1, \cdots, t-1\}$.
Directly minimizing the empirical cost function defined in \Cref{eq:empirical_cost_function} requires re-estimations on all pairs of $\mathbf{r}_{i}$ and $\mathbf{s}_{i}$. 
Instead, the subspace basis $\mathbf{L}_{t}$ is updated by minimizing a surrogate function of the empirical cost function $g_{t}(\mathbf{L}_{t})$, which provides an upper bound for $f_{t}(\mathbf{L}_{t})$ so that $g_{t}(\mathbf{L}_{t}) > f_{t}(\mathbf{L}_{t})$. 
We define the surrogate function as
\begin{equation}
\label{eq:surrogate_function}
\begin{aligned}
g_{t}(\mathbf{L}_{t}) 
= & \frac{1}{t} \sum_{t=1}^{\infty} \hat{\ell}(\mathbf{d}_{t}, \mathbf{L}_{t}, \mathbf{r}_{t}, \mathbf{s}_{t})
+ \frac{\lambda_{1}}{2t} \left\Vert \mathbf{L}_{t} \right\Vert_{F}^{2}\\
= & \frac{1}{t} \sum_{i=1}^{t}
( 
\frac{1}{2} \left\Vert \mathbf{d}_{i} - \mathbf{L}_{t} \mathbf{r}_{i} - \mathbf{s}_{i} \right\Vert_{2}^{2}
+ \frac{\lambda_{1}}{2} \left\Vert \mathbf{r}_{i} \right\Vert_{2}^{2} \\
& + \lambda_{2} \left\Vert \mathbf{s}_{i} \right\Vert_{\ell_{1}/\ell_{\infty}}
) 
+ \frac{\lambda_{1}}{2t} \left\Vert \mathbf{L}_{t} \right\Vert_{F}^{2}.
\end{aligned}
\end{equation}
The minimization of $g_{t}(\mathbf{L}_{t})$ with respect to $\mathbf{L}_{t}$ is termed as Subspace Basis Update, which is then explained in \Cref{subsec:solve_L}. 
The entire O-LSD algorithm is summarized in \Cref{alg:online_alg}.

\begin{algorithm}[t]
	\caption{Proposed O-LSD Algorithm for MOD}
	\label{alg:online_alg}
	\begin{algorithmic}[1]
		\renewcommand{\algorithmicrequire}{\textbf{Input:}}
		\renewcommand{\algorithmicensure}{\textbf{Output:}}
		\REQUIRE  $\mathbf{d}_{t} \in \mathbb{R}^{p}$, $\mathbf{L}_{t-1} \in \mathbb{R}^{p \times r}$, $\mathbf{A}_{t-1}$ and $\mathbf{B}_{t-1}$
		\ENSURE  $\mathbf{b}_{t}$, $\mathbf{r}_{t}$, $\mathbf{s}_{t}$ and $\mathbf{L}_{t}$
		
		\STATE Separate the foreground and background:
		\begin{equation*}
		\begin{aligned}
		(\mathbf{r}_{t}^{*}, \mathbf{s}_{t}^{*}) = \argmin_{\mathbf{r}_{t}, \mathbf{s}_{t}}
		& \frac{1}{2} \left\Vert \mathbf{d}_{t}- \mathbf{L}_{t-1} \mathbf{r}_{t} - \mathbf{s}_{t} \right\Vert^{2}_{2} \\
		& +  \frac{\lambda_{1}}{2} \left\Vert \mathbf{r}_{t}  \right\Vert_{2}^{2} 
		+ \lambda_{2} \left\Vert \mathbf{s}_{t} \right\Vert_{\ell_{1}/\ell_{\infty}},
		\end{aligned}
		\end{equation*}
		which is solved by \Cref{alg:frame_r_s}.
		
		\STATE Compute the background frame: $\mathbf{b}_{t} = \mathbf{L}_{t-1} \mathbf{r}_{t}$.
		
		\STATE Update the accumulation matrices $\mathbf{A}_{t} $ and $\mathbf{B}_{t}$ by \Cref{eq:A_B}.
		
		\STATE Update the subspace basis $\mathbf{L}_{t}$:
		\begin{equation*}
		\begin{aligned}
		\mathbf{L}_{t}^{*} 
		= &\argmin_{\mathbf{L}_{t}} \trace(\mathbf{L}_{t}^{T}(\lambda_{1} \mathbf{I} + \mathbf{A}_{t}) \mathbf{L}_{t}) 
		- 2 \trace(\mathbf{L}_{t}^T \mathbf{B}_{t}),
		\end{aligned}
		\end{equation*}
		whose solution is presented in \Cref{alg:online_update_basis}.
		
		\RETURN $\mathbf{b}_{t}$, $\mathbf{r}_{t}$, $\mathbf{s}_{t}$ and $\mathbf{L}_{t}$.
	\end{algorithmic} 
\end{algorithm}

\subsection{Foreground and Background Separation}
\label{subsec:solve_r_s}

Foreground and Background Separation obtains a pair of $\mathbf{r}_{t}^{*}$ and $\mathbf{s}_{t}^{*}$ by solving the optimization problem defined in \Cref{eq:solve_r_c}, where $\mathbf{L}_{t-1}$ is provided by the previous time instance $t-1$.
As $[\mathbf{L} \ \mathbf{I}]^{T}[\mathbf{L} \ \mathbf{I}]$ is always positive semi-definite, the objective function of \Cref{eq:solve_r_c} is convex with respect to $(\mathbf{r}_{t}, \mathbf{s}_{t})$.
For solving this convex optimization problem, instead of solving $\mathbf{r}_{t}$ and $\mathbf{s}_{t}$ together, we adopt a Block Coordinate Descent (BCD) method \cite{wright2015CD}, where $\mathbf{r}_{t}$ and $\mathbf{s}_{t}$ are alternatingly updated by fixing each other.

By fixing $\mathbf{s}_{t}$, $\mathbf{r}_{t}^{*}$ is obtained by solving
\begin{equation}
\begin{aligned}
\mathbf{r}_t^{*} = \argmin_{r_t} 
\frac{1}{2} \left\Vert \mathbf{d}_{t} - \mathbf{L}_{t-1} \mathbf{r}_t - \mathbf{s}_t \right\Vert_{2}^{2} 
+ \frac{\lambda_{1}}{2} \left\Vert \mathbf{r}_t \right\Vert_{2}^{2},
\end{aligned}
\end{equation}
which constructs a least-square problem, and its closed-form solution is given by
\begin{equation}
\label{eq:solve_r}
\begin{aligned}
\mathbf{r}_{t}^{*} = (\mathbf{L}_{t-1}^{T} \mathbf{L}_{t-1} + \lambda_{1}I)^{-1}(\mathbf{d}_{t} - \mathbf{s}_t).
\end{aligned}
\end{equation}

Using fixed $\mathbf{r}_{t}$, let $\mathbf{u} = \mathbf{d}_{t} - \mathbf{L}_{t-1} \mathbf{r}_{t}$, then the sub-problem for estimating the structured sparsity $\mathbf{s}_{t}$ is defined as
\begin{equation}
\begin{aligned}
\mathbf{s}_{t}^{*} = \argmin_{\mathbf{s}_{t}} \frac{1}{2}  \left\Vert \mathbf{u} - \mathbf{s}_{t} \right\Vert^{2}_{2}
+  \lambda_{2} \left\Vert \mathbf{s}_{t} \right\Vert_{\ell_{1}/\ell_{\infty}},
\end{aligned}
\end{equation}
whose solution is obtained by its dual problem that define a Quadratic Min-Cost Flow problem \cite{mairal2010online_dictionary_learning,mairal2010proxflow,mairal2011proximal_flow_review} as
\begin{equation}
\label{eq:ProxFlow_dual}
\begin{aligned}
\xi^{*} = \argmin_{\xi}  & \frac{1}{2} \left\Vert \mathbf{u} - \sum_{g \in \mathcal{G}} \xi^{g}\right\Vert_{2}^{2} \\
s.t. \ 
& \forall g \in \mathcal{G},  \left\Vert \xi^{g} \right\Vert_{1} \le \lambda_{2} \text{ and } \xi^{g}_{j} = 0 \ if \  j \notin g
\end{aligned},
\end{equation}
where $\xi^{g} \in \mathbb{R}^{p}, \forall{g} \in \mathcal{G}$ denotes the corresponding dual variables for the group of variables in $g$, and $\xi$ is the set of all $\xi^{g}, \forall g \in \mathcal{G}$. 
The primal solution $\mathbf{s}_t$ is then obtained by 
\begin{equation}
\label{eq:ProxFlow_dual_to_prime}
\mathbf{s}_{t}^{*} = \mathbf{u} - \sum_{g \in \mathcal{G}}\xi^{*g}.
\end{equation}

The iteration of alternatingly estimation of $\mathbf{r}_{t}$ and $\mathbf{s}_{t}$ continues until the stop criterion is reached:
\begin{equation}
\label{eq:online_stop_criterion}
\frac{\max \{\left\Vert \mathbf{r}_{t}^{\prime} - \mathbf{r}_{t}^{\dprime} \right\Vert_{2}, \left\Vert \mathbf{s}_{t}^{\prime} - \mathbf{s}_{t}^{\dprime} \right\Vert_{2} \} }{p} \le \tau,
\end{equation}
where $(\mathbf{r}_{t}^{\prime}, \mathbf{s}_{t}^{\prime})$ and $(\mathbf{r}_{t}^{\dprime}, \mathbf{s}_{t}^{\dprime})$ are two pairs of estimation solutions at two consecutive iterations.
Similar to \cite{shakeri2016COROLA,javed2019sssr}, the stop criterion is set as $\tau = \num{1.0e-5}$ in this paper. 
The BCD algorithm for Foreground and Background Separation is summarized in \Cref{alg:frame_r_s}.

\begin{algorithm}[t]
	\caption{Block Coordinate Descent Method for Foreground and Background Separation}
	\label{alg:frame_r_s}
	\begin{algorithmic}[1]
		\renewcommand{\algorithmicrequire}{\textbf{Input:}}
		\renewcommand{\algorithmicensure}{\textbf{Output:}}
		\REQUIRE $\mathbf{L}_{t-1} \in \mathbb{R}^{p*r}$, $\lambda_{1} > 0$ and $\lambda_{2} > 0$
		\ENSURE  $\mathbf{r}_{t}$ and $\mathbf{s}_{t}$
		
		\STATE $\mathbf{r}_{t} = \mathbf{0}$, $\mathbf{s}_{t} = \mathbf{0}$.
		
		\WHILE{not converged}
		\STATE Estimate $\mathbf{r}_{t}$: 
		\begin{equation*}
		\begin{aligned}
		\mathbf{r}_{t} = (\mathbf{L}_{t-1}^{T} \mathbf{L}_{t-1} + \lambda_{1}I)^{-1}(\mathbf{d}_{t} - \mathbf{s}_t).
		\end{aligned}
		\end{equation*}
		
		\STATE Estimate $\mathbf{s}_{t}$ by solving its dual problem on $\xi$ defined in \Cref{eq:ProxFlow_dual}, then
		\begin{equation*}
		\begin{aligned}
		\mathbf{s}_{t} = \mathbf{d}_{t} - \mathbf{L}_{t-1} \mathbf{r}_{t} - \sum_{g \in \mathcal{G}}\xi^{*g}
		\end{aligned}.
		\end{equation*}

		\STATE Check the convergence using \Cref{eq:online_stop_criterion}.
		\ENDWHILE

		\RETURN $\mathbf{r}_{t}$ and $\mathbf{s}_{t}$
	\end{algorithmic} 
\end{algorithm}

\subsection{Subspace Basis Update}
\label{subsec:solve_L}

After estimating $(\mathbf{r}_{t}, \mathbf{s}_{t})$, the subspace basis $\mathbf{L}_{t}$ is updated by minimizing the surrogate function of empirical cost function $g_{t}(\mathbf{L}_{t})$, which defines a optimization problem as
\begin{equation}
\label{eq:sub_problem_L}
\begin{aligned}
\mathbf{L}_{t}^{*} 
= &\argmin_{\mathbf{L}_{t}} g_{t}(\mathbf{L}_{t}) \\
= &\argmin_{\mathbf{L}_{t}} \trace(\mathbf{L}_{t}^{T}(\lambda_{1} \mathbf{I} + \mathbf{A}_{t}) \mathbf{L}_{t}) 
- 2 \trace(\mathbf{L}_{t}^T \mathbf{B}_{t}),
\end{aligned}
\end{equation}
in which $\trace(\cdot)$ denotes the trace of a matrix, and $\mathbf{A}_{t}$ and $\mathbf{B}_{t}$ are two auxiliary accumulation matrices that are introduced to remove duplicated calculations at each time instance,
\begin{equation}
\label{eq:A_B}
\begin{aligned}
\begin{cases}
\mathbf{A}_{t} &= \mathbf{A}_{t-1} + \mathbf{r}_{t} \mathbf{r}_{t}^T \\
\mathbf{B}_{t} &= \mathbf{A}_{t-1} +(\mathbf{d}_{t} - \mathbf{s}_{t}) \mathbf{r}_{t}^T
\end{cases}.
\end{aligned}
\end{equation}
Similar to \cite{feng2013ORPCA_SO,shakeri2016COROLA,shen2016online_subspace_clustering}, the optimization problem defined in \Cref{eq:sub_problem_L} is solved by a Block Coordinate Descent Method for avoiding matrix inverse of large matrix. 
The Subspace Basis Update algorithm is illustrated in \Cref{alg:online_update_basis} .

\begin{algorithm}
	\caption{Block Coordinate Descent Method for Subspace Basis Update}
	\label{alg:online_update_basis}
	\begin{algorithmic}[1]
		\renewcommand{\algorithmicrequire}{\textbf{Input:}}
		\renewcommand{\algorithmicensure}{\textbf{Output:}}
		\REQUIRE $\mathbf{L}_{t-1}=[\mathbf{l}_{1}, \cdots, \mathbf{l}_{r}] \in \mathbb{R}^{p*r}$, $\mathbf{A}_{t}=[\mathbf{a}_{1}, \cdots, \mathbf{a}_{r}] \in \mathbb{R}^{r*r}$, $\mathbf{B}_{t}=[\mathbf{b}_{1}, \cdots, \mathbf{b}_{r}]  \in \mathbb{R}^{p*r}$ and $\lambda_{1} > 0$
		\ENSURE  $\mathbf{L}_{t}$
		\STATE $\tilde{\mathbf{A}} =  \mathbf{A}_{t}+ \lambda_{1}\mathbf{I}$.
		\FOR {$i = 1$ to $r$}
		\STATE $\mathbf{l}_{i}= \frac{1}{\tilde{\mathbf{A}}_{i, i}} (\mathbf{b}_{i} - \mathbf{L}_{t-1} \mathbf{a}_{i} ) + \mathbf{l}_{i}$.
		\ENDFOR
		\STATE $\mathbf{L}_{t}=[\mathbf{l}_{1}, \cdots, \mathbf{l}_{r}]$.
		\RETURN $\mathbf{L}_{t}$
	\end{algorithmic} 
\end{algorithm}

The subspace basis $\mathbf{L}_{0}$ is initialized before starting O-LSD. 
$\mathbf{L}_{0}$ can be initialized by either the first a few frames in the given sequence or their Principal Components \cite{javed2019sssr,xu2013GOSUS} . 
In satellite videos, moving objects move slowly, and choosing these initialization scheme risks including the slow moving foreground objects into the background, which thus influences the detection performance, or a more extended sequence for initialization is required. 
Such a satellite video with  adequate length is, however, not available technically yet.
Therefore, in satellite videos, we recommend initializing $\mathbf{L}_{0}$ by random values instead, which performs pretty well in practice.

\subsection{Convergence Analysis}
\label{subsec:convergence}

One technical contribution of this paper is to present the proposed O-LSD algorithm converges almost surely under mild condition. 

\begin{ass}
	The observed data are uniformly bounded, and each data is independent.
\end{ass}

As a widely-used assumption \cite{mairal2010online_dictionary_learning, feng2013ORPCA_SO, shen2016online_subspace_clustering}, this assumption on the boundedness of observation data is quiet natural for real videos.
Based the above assumption, we present our first conclusion on the convergence of the surrogate function $g_{t}(\mathbf{L}_{t})$.

\begin{thm}
\label{thm:surrogate_function_convergence}
Let $\{\mathbf{L}_{t} \}_{t=1}^{\infty}$ be the sequence of solution obtained by \Cref{alg:online_alg}, the surrogate function $g_{t}(\mathbf{L_t})$ converges almost surely. 
\end{thm}

Similarly, we obtain the convergence of two solutions obtained at two consecutive time instance by \Cref{alg:online_alg}.

\begin{thm}
\label{thm:solution_convergence}
For two solutions produced by \Cref{alg:online_alg} at two consecutive time instances, $\left\Vert \mathbf{L}_{t} - \mathbf{L}_{t+1} \right\Vert_{F} = O(\frac{1}{t})$.
\end{thm}

Then, we analyze the gap between the empirical cost function $f_{t}(\mathbf{L}_{t})$ and its surrogate function $g_{t}(\mathbf{L}_{t})$ with the estimated $\mathbf{L}_t$.

\begin{thm}
\label{thm:surrogate_function_converge_to_emiprical_cost_function}
Note $f_{t}(\mathbf{L})$ is the empirical cost function defined in \Cref{eq:empirical_cost_function}, and $g_{t}(\mathbf{L})$ is its surrogate function defined in \Cref{eq:surrogate_function}. $\mathbf{L}_{t}$ is the solution obtained by \Cref{alg:online_alg}, when $t$ tends to infinity, $g_t(\mathbf{L}_{t}) - f_t(\mathbf{L}_{t})$ converges to 0 almost surely.
\end{thm}

In stochastic optimization, the expected cost function over $\mathbf{L}$ is defined as
\begin{equation}
f(\mathbf{L}) = \mathbb{E}_{\mathbf{d}} [ \ell(\mathbf{d}, \mathbf{L}) ] = \lim_{t \to \infty} f_{t}(\mathbf{L}),
\end{equation}
then we present the convergence of the gap between the expect cost function $f(\mathbf{L}_{t})$ and the surrogate function $g_{t}(\mathbf{L}_{t})$.

\begin{thm}
\label{thm:surrogate_function_converge_to_expected_cost_function}
As $t$ tends to infinity, given the $\mathbf{L}_{t}$ is obtained by \Cref{alg:online_alg}, $g_{t}(\mathbf{L}_{t}) - f(\mathbf{L}_{t})$ converges to 0 almost surely.
\end{thm}

Furthermore, the solution $\mathbf{L}_{t}$ obtained by \Cref{alg:online_alg} is not a stationary point of expected cost function $f(\mathbf{L})$, when $t$ tends to infinity, which on contrary is proved true in \cite{mairal2010online_dictionary_learning,feng2013ORPCA_SO, shen2016online_subspace_clustering}.
Due to the existence of more than one solutions to \Cref{eq:solve_r_c}, $\hat{\ell}(\mathbf{d}, \mathbf{L}, \mathbf{r}, \mathbf{s})$ is no longer strictly convex (or strongly convex) with respect to $(\mathbf{r}, \mathbf{s})$, and the gradient of the expected cost function $\nabla_{\mathbf{L}} f(\mathbf{L})$ is no longer Lipschitz.
Therefore, the gradient of the expected cost function $\nabla_{\mathbf{L}} f(\mathbf{L})$ would not become zero when $t$ tends to infinity, based on which we conclude the solution $\mathbf{L}_{t}$ may not be the stationary point of the expected cost function as $t$ tends infinity.

Please refer to the appendices for detailed proofs of the presented theorems.

\section{Experiments}
\label{sec:experiments}

The detection performance of O-LSD was evaluated on a dataset of two satellite videos. 
This dataset is constructed from a satellite video captured over Las Vagas, USA on March 25, 2014, whose spatial resolution is 1.0 meter and the frame rate is 30 frames per second. 
Both videos contains 700 frame with boundary boxes for moving vehicles as groundtruth, and details on both videos are listed in \Cref{tbl:dataset_info} 
\footnote[2]{Moving vehicles are manually labeled by the Computer Vision Annotation Tool (CVAT), and a boundary box is provided for each moving object on each frame.}.
In this paper, we used the first 200 frames in each video for the discussion on parameter selection, and the remaining frames were utilized for performance evaluation against existing state-of-the-art methods.

\begin{table}[h]
	\caption{Information on the evaluation datasets}
	\label{tbl:dataset_info}
	\centering
	\begin{tabular}{c|c|c|c|c|c}
		\hline
		\multirow{2}{*}{Video} & \multirow{2}{*}{Frame Size} & \multicolumn{2}{c|}{Cross Validation} & \multicolumn{2}{c}{Performance Evaluation} \\
		\cline{3-6}
		& & \#Frames & \#Vehicles & \#Frames & \#Vehicles \\
		
		\hline
		001 & $400 \times 400$ & 200 & 9306 & 500 & 18167 \\ 
		\hline
		002 & $600 \times 400$ & 200 & 13443 & 500 & 39362 \\ 
		\hline
	\end{tabular} 
\end{table}

The detection performance on moving object detection is evaluated on recall, precision and $F_{1}$ scores given by
\begin{equation}
\begin{aligned}
&\text{recall} = TP/(TP + FN) \\
&\text{precision} = TP/(TP + FP) \\
&F_1 =  \frac{2 \times \text{recall} \times \text{precision}}{\text{recall} + \text{precision}}
\end{aligned},
\end{equation}
where $TP$ denotes the number of correct detections, $FN$ and $FP$ are the numbers of missed detections and false alarms, respectively.
In this paper, we define a correct detection with maximum Intersection over Union (IoU) against the groundtruth greater than a threshold. 
To complement the vehicles in small size in satellite videos, the threshold is set as 0.3
\footnote[3]{The estimated foreground is built by contiguous values rather than binary value, so we deploy threshold segmentation as post-processing for extracting the foreground mask and the moving objects \cite{gao2012block_sparsity_bg}.}.
In this paper, we refer 5-Frame detection performance at each time instance to the metrics obtained from its 5 latest frames, which is used for observing the convergence, and the accumulated detection performance is measured on all frames before current time instances, which is used for comparing the overall performance over a sequence.

\subsection{Parameter Setting}
\label{subsec:parameter_setting}

The performance of O-LSD is controlled by the dimension of the estimated subspace $r$, and two weights for the low-rank term and the structured sparsity penalty separably, $\lambda_1$ and $\lambda_2$. 
The following experiments are conducted on the cross-validation sequence from Video 001.

The weight $\lambda_{1}$ assigns the importance of the low-rank subspace term. 
With the fixed $\lambda_{2}$, a significantly small $\lambda_{1}$ would encourage more information encoded in the low-rank subspace factors and hurt the detection performance in both terms of recall and precision. 
Increasing $\lambda_{1}$ improves the detection performance by the increased emphasis on the low-rank subspace modeling. 
After approaching the best detection performance, continuing increasing $\lambda_{1}$ would prevent the information encoded into the background.
As illustrated in \Cref{fig:lambda1_fixed_lambda2}, with the fixed $\lambda_{2} = 0.025$ and $r=5$, as $\lambda_{1}$ increases, the $F_{1}$ score gradually increases to about 75\% from around 60\%, then the detection performance starts dropping with continuously increasing $\lambda_{1}$.

\begin{figure}[t]
	\footnotesize
	\centering
	\subfloat[Varying $\lambda_{1}$ with $\lambda_{2} = 0.025$]{
		\begin{tikzpicture}
		\begin{axis}[ 
		width = 3.5cm,
		height= 3.5cm,
		xlabel={$\lambda_{1}$},
		xmin=0.00025,
		xmax=2.5,
		xmode=log, 
		ymin = 0,
		ymax=85,
		ylabel={$F_1$},
		xlabel near ticks,
		ylabel near ticks,
		label style={font=\tiny},
		tick label style={font=\tiny},
		legend style={ font=\tiny},
		legend pos= south west,
		]	
		\addplot +[black, thick, smooth, mark=none] table[x =lambda1, y expr=\thisrow{f1}*100, col sep=comma] {data/parameter_setting/lambda1_fixed_lambda2_10.csv};
		\end{axis}
		\end{tikzpicture}
		
		\begin{tikzpicture}
		\begin{axis}[ 
		width = 3.5cm,
		height= 3.5cm,
		xlabel={$\lambda_{1}$},
		xmin=0.00025,
		xmax=2.5,
		xmode=log, 
		ymin = 0,
		ymax=85,
		ylabel={Recall},
		xlabel near ticks,
		ylabel near ticks,
		label style={font=\tiny},
		tick label style={font=\tiny},
		legend style={ font=\tiny},
		legend pos= south west,
		]	
		\addplot +[red, thick, smooth, mark=none] table[x =lambda1, y expr=\thisrow{recall}*100, col sep=comma] {data/parameter_setting/lambda1_fixed_lambda2_10.csv};
		\end{axis}
		\end{tikzpicture}
		
		\begin{tikzpicture}
		\begin{axis}[ 
		width = 3.5cm,
		height= 3.5cm,
		xlabel={$\lambda_{1}$},
		xmin=0.00025,
		xmax=2.5,
		xmode=log, 
		ymin = 0,
		ymax=85,
		ylabel={Precision},
		xlabel near ticks,
		ylabel near ticks,
		label style={font=\tiny},
		tick label style={font=\tiny},
		legend style={ font=\tiny},
		legend pos= south west,
		]		
		\addplot +[blue, thick, smooth, mark=none] table[x =lambda1, y expr=\thisrow{precision}*100, col sep=comma] {data/parameter_setting/lambda1_fixed_lambda2_10.csv};
		\end{axis}
		\end{tikzpicture}
		\label{fig:lambda1_fixed_lambda2}
	} \\
	\subfloat[Varying $\lambda_{2}$ with $\lambda_{1} = 0.0025$]{
		\begin{tikzpicture}
		\begin{axis}[ 
		width = 3.5cm,
		height= 3.5cm,
		xlabel={$\lambda_{2}$},
		xmin=0.00001,
		xmax=2.5,
		xmode=log, 
		ymin = 0,
		ymax=85,
		ylabel={$F_1$},
		xlabel near ticks,
		ylabel near ticks,
		label style={font=\tiny},
		tick label style={font=\tiny},
		legend style={ font=\tiny},
		legend pos= south west,
		]	
		\addplot +[black, thick, smooth, mark=none] table[x =lambda2, y expr=\thisrow{f1}*100, col sep=comma] {data/parameter_setting/lambda2_fixed_lambda1.csv};
		\end{axis}
		\end{tikzpicture}
		
		\begin{tikzpicture}
		\begin{axis}[ 
		width = 3.5cm,
		height= 3.5cm,
		xlabel={$\lambda_{2}$},
		xmin=0.00001,
		xmax=2.5,
		xmode=log, 
		ymin = 0,
		ymax=85,
		ylabel={Recall},
		xlabel near ticks,
		ylabel near ticks,
		label style={font=\tiny},
		tick label style={font=\tiny},
		legend style={ font=\tiny},
		legend pos= south west,
		]	
		\addplot +[red, thick, smooth, mark=none] table[x =lambda2, y expr=\thisrow{recall}*100, col sep=comma] {data/parameter_setting/lambda2_fixed_lambda1.csv};
		\end{axis}
		\end{tikzpicture}
		
		\begin{tikzpicture}
		\begin{axis}[ 
		width = 3.5cm,
		height= 3.5cm,
		xlabel={$\lambda_{2}$},
		xmin=0.00001,
		xmax=2.5,
		xmode=log, 
		ymin = 0,
		ymax=85,
		ylabel={Precision},
		xlabel near ticks,
		ylabel near ticks,
		label style={font=\tiny},
		tick label style={font=\tiny},
		legend style={ font=\tiny},
		legend pos= south west,
		]		
		\addplot +[blue, thick, smooth, mark=none] table[x =lambda2, y expr=\thisrow{precision}*100, col sep=comma] {data/parameter_setting/lambda2_fixed_lambda1.csv};
		\end{axis}
		\end{tikzpicture}
		\label{fig:lambda2_fixed_lambda1}
	}
	
	\caption{Moving object detection performance with different $\lambda_{1}$ and $\lambda_{2}$.}
	\label{fig:lambda1_lambda2}
\end{figure}
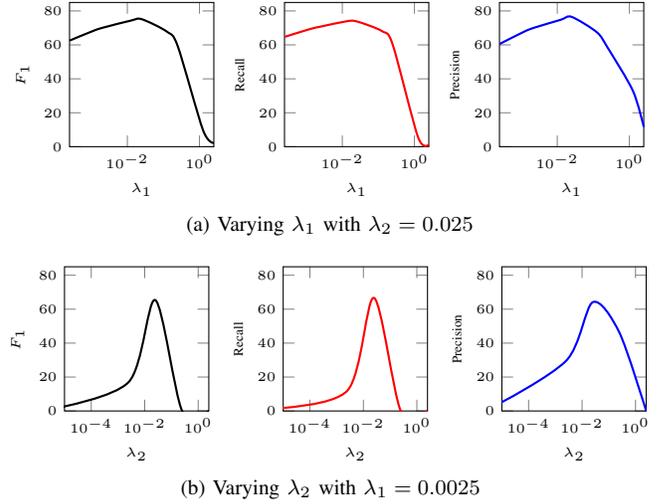

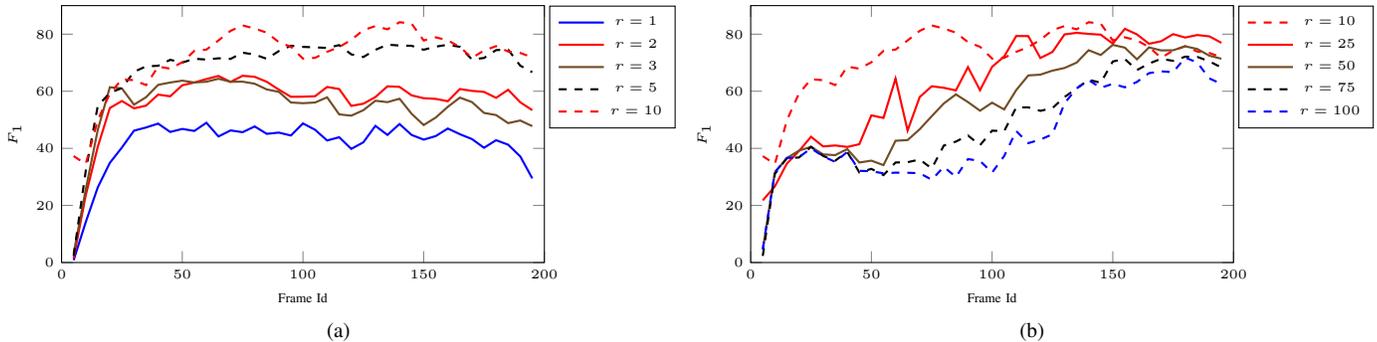
\begin{figure*}[t]
	\footnotesize
	\centering
	\subfloat[]{
		\begin{tikzpicture}
		\begin{axis}[ 
		width = 8.0cm,
		height= 5.0cm,
		xlabel={Frame Id},
		xmin=0,
		xmax=200,
		ymin = 0,
		ymax=90,
		ylabel={$F_1$},
		xlabel near ticks,
		ylabel near ticks,
		label style={font=\tiny},
		tick label style={font=\tiny},
		legend style={at={(1.01,1)},anchor=north west, font=\tiny},
		]	
		\addplot +[thick, mark=none ] table[x=frame_id, y expr=\thisrow{f1}*100, col sep=comma] {data/random/MODE_RANDOM_RANK_1_INITIAL_FRAME_0_FACTOR_0.1.avg.csv};
		\addlegendentry{$r=1$}
		\addplot +[thick, mark=none ] table[x=frame_id, y expr=\thisrow{f1}*100, col sep=comma] {data/random/MODE_RANDOM_RANK_2_INITIAL_FRAME_0_FACTOR_0.1.avg.csv};
		\addlegendentry{$r=2$}
		\addplot +[thick, mark=none ] table[x=frame_id, y expr=\thisrow{f1}*100, col sep=comma] {data/random/MODE_RANDOM_RANK_3_INITIAL_FRAME_0_FACTOR_0.1.avg.csv};
		\addlegendentry{$r=3$}
		\addplot +[thick, dashed, mark=none ] table[x=frame_id, y expr=\thisrow{f1}*100, col sep=comma] {data/random/MODE_RANDOM_RANK_5_INITIAL_FRAME_0_FACTOR_0.1.avg.csv};
		\addlegendentry{$r=5$}
		\addplot +[thick, dashed, mark=none, red] table[x=frame_id, y expr=\thisrow{f1}*100, col sep=comma] {data/random/MODE_RANDOM_RANK_10_INITIAL_FRAME_0_FACTOR_0.1.avg.csv};
		\addlegendentry{$r=10$}
		\end{axis}
		\end{tikzpicture}
	}
	\subfloat[]{
		\begin{tikzpicture}
		\begin{axis}[ 
		width = 8.0cm,
		height= 5.0cm,
		xlabel={Frame Id},
		xmin=0,
		xmax=200,
		ymin = 0,
		ymax=90,
		ylabel={$F_1$},
		xlabel near ticks,
		ylabel near ticks,
		label style={font=\tiny},
		tick label style={font=\tiny},
		legend style={at={(1.01,1)},anchor=north west, font=\tiny},
		]	
		\addplot +[thick, dashed, mark=none, red] table[x=frame_id, y expr=\thisrow{f1}*100, col sep=comma] {data/random/MODE_RANDOM_RANK_10_INITIAL_FRAME_0_FACTOR_0.1.avg.csv};
		\addlegendentry{$r=10$}
		\addplot +[thick, mark=none] table[x=frame_id, y expr=\thisrow{f1}*100, col sep=comma] {data/random/MODE_RANDOM_RANK_25_INITIAL_FRAME_0_FACTOR_0.1.avg.csv};
		\addlegendentry{$r=25$}
		\addplot +[thick, mark=none] table[x=frame_id, y expr=\thisrow{f1}*100, col sep=comma] {data/random/MODE_RANDOM_RANK_50_INITIAL_FRAME_0_FACTOR_0.1.avg.csv};
		\addlegendentry{$r=50$}
		\addplot +[thick, dashed, mark=none] table[x=frame_id, y expr=\thisrow{f1}*100, col sep=comma] {data/random/MODE_RANDOM_RANK_75_INITIAL_FRAME_0_FACTOR_0.1.avg.csv};
		\addlegendentry{$r=75$}
		\addplot +[thick, dashed, mark=none] table[x=frame_id, y expr=\thisrow{f1}*100, col sep=comma] {data/random/MODE_RANDOM_RANK_100_INITIAL_FRAME_0_FACTOR_0.1.avg.csv};
		\addlegendentry{$r=100$}
		\end{axis}
		\end{tikzpicture}
	}
	\caption{5-Frame performance evaluation of O-LSD with different $r$ on cross-validation sequence from Video 001.} 
	\label{fig:random_initialization_sq_1}
\end{figure*}

\begin{table*}[t]
	\caption{Detection Performance Comparison with Online Algorithms}
	\label{tbl:comparison_online}
	\centering
	\begin{tabular}{c|ccc|ccc|c}		
		\hline 
		\multirow{2}{*}{Video} & \multicolumn{3}{c|}{001} & \multicolumn{3}{c|}{002} & \multirow{2}{*}{Avg($F_1$)}\\
		\cline{2-7}
		& Recall &  Precision&  $F_{1}$ score  & Recall &  Precision&  $F_{1}$ score  & \\ 
		\hline 
		GRASTA	& \textbf{76.96\%} & 31.76\% & 44.97\% & \underline{72.22\%} &46.73\% & 56.75\% &  50.86\% \\ 
		\hline 
		
		OR-PCA	& \underline{66.51\%} & 41.50\%  & 51.11\% & 71.94\% & \underline{73.79\%}  & \underline{72.86\%} &  \underline{61.99\%} \\ 
		\hline 
		
		GUSOS	& 59.35\% & \underline{49.15\%}  & 53.77\% & 68.03\%  & 62.61\%  &  65.21\% & 59.49\% \\
		\hline 
		\textbf{O-LSD} & 64.99\% & \textbf{63.75\%}  &  \textbf{64.36\%} & \textbf{73.00}\% & \textbf{90.21\%}  &  \textbf{80.69}\% &  \textbf{72.48\%} \\
		\hline
	\end{tabular} 
\end{table*}

Increasing $\lambda_{2}$ with fixed $\lambda_{1}$ would put more emphasis on the structured sparsity of the extracted foreground, which thus improves the precision of the detected moving objects by restraining the random noises. 
When $\lambda_{2}$ continuously increases, the weight for the structured sparsity norm tends too large to encode information into the foreground, which then decreases the detection performance. 
As illustrated in \Cref{fig:lambda2_fixed_lambda1}, with fixed $\lambda_{1} = 0.0025$ and $r=5$, the $F_{1}$ score first approaches to the highest point as $\lambda_{2}$ increases, then the same metric drops when $\lambda_{2}$ tends too large.

Then we discuss the selection of the dimension of the subspace $r$.
With fixed $\lambda_{1}$ and $\lambda_{2}$, the O-LSD with smaller $r$ probably converges faster, however, it may fail in modeling the permutation of the background for a long video sequence.
On the contrary, selecting a higher $r$ would disadvantage the updating of subspace basis and require more frames before O-LSD converges. 
As presented in \Cref{fig:random_initialization_sq_1}, for $r \le 10$, the 5-frame $F_1$ scores increase faster than those with $r \ge 25$.
As the sequence length increases, 5-frame $F_1$ scores by the O-LSD with $r \le 10$ show a trend of dropping. 
For $r=25$, this trend is negligible, and the same metric is still rising for $r > 25$, which means the O-LSD requires a longer sequence to converge.
The highest $F_{1}$ score at Frame-200 is achieved by $r=25$ in \Cref{fig:random_initialization_sq_1}.

In the rest of the paper, based on cross-evaluation, we select $\lambda_{1} = 1 / {\sqrt{p}}$ and $\lambda_{1} / \lambda_{2} = 0.1$ with $r = 25$ for evaluation, and further fine-tuning on the parameter selection would improve the detection performance by O-LSD.

\subsection{Comparison with Online Approach}

To verify the effectiveness of O-LSD, detection results by O-LSD are compared against the state-of-the-art online approaches, GRASTA \cite{he2012GRASTA}, OR-PCA \cite{feng2013ORPCA_SO} and GOSUS \cite{xu2013GOSUS}. 
For all methods, the subspace is initialized by the random scheme, and their parameters are selected through cross-validation.

O-LSD boosts the detection performance in terms of the precision and $F_{1}$ scores.
This improvement should be owned to the structured sparsity penalty, which suppress the random noises in estimated foreground frames. 
As present in \Cref{tbl:comparison_online}, the O-LSD achieves the highest precision and $F_{1}$ scores on both videos, although there is a little drop of recall score on Video 001, compared with GRASTA and OR-PCA. 
In terms of online methods employing structured sparsity penalty, O-LSD outperforms GOSUS on both satellite videos.
For common moving object detection tasks GOSUS is proved effective, however, it produces no improvement to OR-PCA on the satellite videos.

Another advantage of O-LSD is its faster convergence against other state-of-the-art online algorithms. 
As demonstrated in \Cref{fig:framewise_comparison}, on both videos, the O-LSD achieves the higher $F_{1}$ scores earlier than GRASTA, OR-PCA and GOSUS, implying that O-LSD converges faster by introducing the structured sparsity penalty. 
Similar trends can also be observed from the perspective of accumulated detection performance, as the accumulated detection performance of O-LSD is always better than the three existing methods as the length of sequence increases, as illustrated in \Cref{fig:accumulated_comparison}.

Besides, as more frames are processed by O-LSD algorithm, the cleanness of the estimated background frame gradually improves, as shown in \Cref{fig:detection_visualization_sq_1}.

\begin{table}
	\caption{Detection Performance by O-LSD with Temporally Down-sampling}
	\label{tbl:temporally_downsampling}
	\centering
	\begin{tabular}{c|ccccc}		
		\hline 
		Video & $T$ & Recall & Precision & $F_{1}$ & Time per Frame\\
		\hline 
		
		\multirow{4}{*}{001} & 1 & 64.99\% & 63.75\%  &  64.36\% & 6.57s \\
		& 3	& 67.85\% & 59.79\% &63.57\% & 12.74s \\
		& 5	& 69.22\% & 62.02\% &65.42\% & 14.62s \\
		& 10	& 67.72\% & 62.14\% & 64.91\% & 14.26s \\
		\hline
		
		\multirow{4}{*}{002} & 1 & 73.00\% & 90.21\%  &  80.69\% & 10.75s \\
		& 3	& 74.21\% & 88.46\% & 80.71\% & 14.48s \\
		& 5	& 74.41\% & 87.72\%  & 80.52\% & 17.86s \\
		& 10 & 73.46\% & 88.68\% & 80.36\% &21.12s  \\
		\hline
	\end{tabular} 
\end{table}

\begin{figure*}
\footnotesize
\centering
\subfloat[Video 001]{
	\begin{tikzpicture}
	\begin{axis}[ 
	width = 8.5cm,
	height= 5.0cm,
	xlabel={Frame Id},
	xmin=0,
	xmax=500,
	ymin = 0,
	ymax=80,
	ylabel={$F_1$},
	xlabel near ticks,
	ylabel near ticks,
	label style={font=\tiny},
	tick label style={font=\tiny},
	legend style={ font=\tiny},
	legend pos= south east,
	]	
	\addplot +[mark=none] table[x expr=\thisrow{frame_id}-200, y expr=\thisrow{f1}*100, col sep=comma] {data/eval/olsd_sq_1.csv};
	\addlegendentry{O-LSD}
	\addplot +[mark=none] table[x expr=\thisrow{frame_id}-200, y expr=\thisrow{f1}*100, col sep=comma] {data/eval/gosus_sq_1.csv};
	\addlegendentry{GOSUS}
	\addplot +[mark=none] table[x expr=\thisrow{frame_id}-200, y expr=\thisrow{f1}*100, col sep=comma] {data/eval/orpca_sq_1.csv};
	\addlegendentry{OR-PCA}
	\addplot +[mark=none] table[x expr=\thisrow{frame_id}-200, y expr=\thisrow{f1}*100, col sep=comma] {data/eval/grasta_sq_1.csv};
	\addlegendentry{GRASTA}

	\end{axis}
	\end{tikzpicture}
}
\subfloat[Video 002]{
	\begin{tikzpicture}
	\begin{axis}[ 
	width = 8.5cm,
	height= 5.0cm,
	xlabel={Frame Id},
	xmin=0,
	xmax=500,
	ymin = 0,
	ymax=90,
	ylabel={$F_1$},
	xlabel near ticks,
	ylabel near ticks,
	label style={font=\tiny},
	tick label style={font=\tiny},
	legend style={ font=\tiny},
	legend pos= south east,
	]	
	\addplot +[mark=none] table[x expr=\thisrow{frame_id}-200, y expr=\thisrow{f1}*100, col sep=comma] {data/eval/olsd_sq_2.csv};
	\addlegendentry{O-LSD}
	\addplot +[mark=none] table[x expr=\thisrow{frame_id}-200, y expr=\thisrow{f1}*100, col sep=comma] {data/eval/gosus_sq_2.csv};
	\addlegendentry{GOSUS}
	\addplot +[mark=none] table[x expr=\thisrow{frame_id}-200, y expr=\thisrow{f1}*100, col sep=comma] {data/eval/orpca_sq_2.csv};
	\addlegendentry{OR-PCA}
	\addplot +[mark=none] table[x expr=\thisrow{frame_id}-200, y expr=\thisrow{f1}*100, col sep=comma] {data/eval/grasta_sq_2.csv};
	\addlegendentry{GRASTA}

	\end{axis}
	\end{tikzpicture}
}
\caption{5-Frame detection performance by different online algorithms.}
\label{fig:framewise_comparison}
\end{figure*}
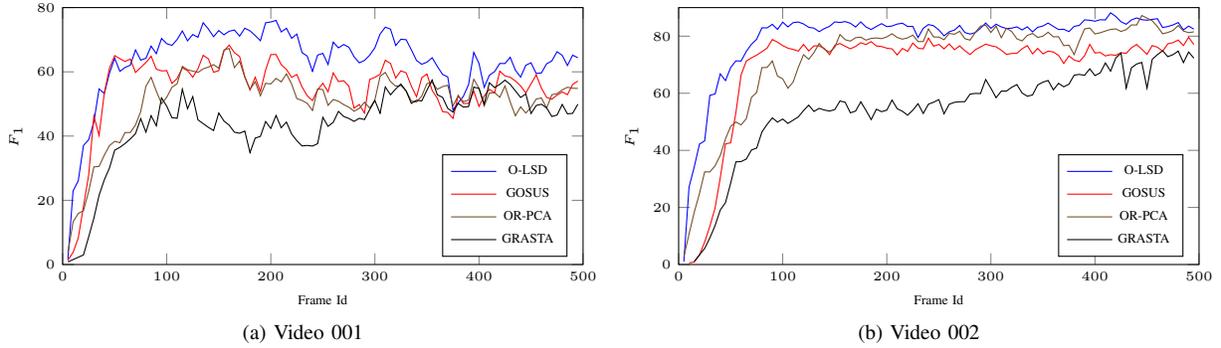

\begin{figure*}
	\footnotesize
	\centering
	\subfloat[Video 001]{
		\begin{tikzpicture}
		\begin{axis}[ 
		width = 8.5cm,
		height= 5.0cm,
		xlabel={Frame Id},
		xmin=0,
		xmax=500,
		ymin = 0,
		ymax=75,
		ylabel={$F_1$},
		xlabel near ticks,
		ylabel near ticks,
		label style={font=\tiny},
		tick label style={font=\tiny},
		legend style={ font=\tiny},
		legend pos= south east,
		]	
		\addplot +[mark=none] table[x expr=\thisrow{frame_id}-200, y expr=\thisrow{f1}*100, col sep=comma] {data/eval/olsd_sq_1.acc.csv};
		\addlegendentry{O-LSD}
		\addplot +[mark=none] table[x expr=\thisrow{frame_id}-200, y expr=\thisrow{f1}*100, col sep=comma] {data/eval/gosus_sq_1.acc.csv};
		\addlegendentry{GOSUS}
		\addplot +[mark=none] table[x expr=\thisrow{frame_id}-200, y expr=\thisrow{f1}*100, col sep=comma] {data/eval/orpca_sq_1.acc.csv};
		\addlegendentry{OR-PCA}
		\addplot +[mark=none] table[x expr=\thisrow{frame_id}-200, y expr=\thisrow{f1}*100, col sep=comma] {data/eval/grasta_sq_1.acc.csv};
		\addlegendentry{GRASTA}

		
		\end{axis}
		\end{tikzpicture}
	}
	\subfloat[Video 002]{
		\begin{tikzpicture}
		\begin{axis}[ 
		width = 8.5cm,
		height= 5.0cm,
		xlabel={Frame Id},
		xmin=0,
		xmax=500,
		ymin = 0,
		ymax=90,
		ylabel={$F_1$},
		xlabel near ticks,
		ylabel near ticks,
		label style={font=\tiny},
		tick label style={font=\tiny},
		legend style={ font=\tiny},
		legend pos= south east,
		]	
		\addplot +[mark=none] table[x expr=\thisrow{frame_id}-200, y expr=\thisrow{f1}*100, col sep=comma] {data/eval/olsd_sq_2.acc.csv};
		\addlegendentry{O-LSD}
		\addplot +[mark=none] table[x expr=\thisrow{frame_id}-200, y expr=\thisrow{f1}*100, col sep=comma] {data/eval/gosus_sq_2.acc.csv};
		\addlegendentry{GOSUS}
		\addplot +[mark=none] table[x expr=\thisrow{frame_id}-200, y expr=\thisrow{f1}*100, col sep=comma] {data/eval/orpca_sq_2.acc.csv};
		\addlegendentry{OR-PCA}
		\addplot +[mark=none] table[x expr=\thisrow{frame_id}-200, y expr=\thisrow{f1}*100, col sep=comma] {data/eval/grasta_sq_2.acc.csv};
		\addlegendentry{GRASTA}

		\end{axis}
		\end{tikzpicture}
	}
	\caption{Accumulated detection performance by different online algorithms.}
	\label{fig:accumulated_comparison}
\end{figure*}
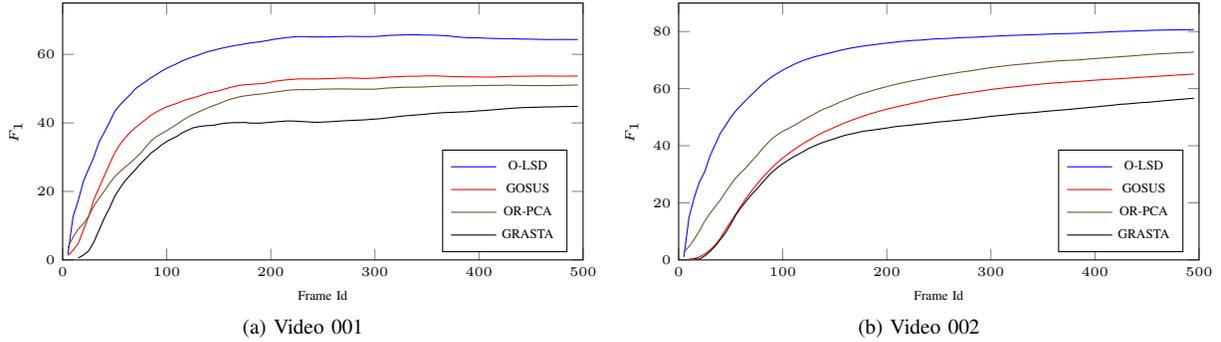

\subsection{Comparison with Batch-based Approach}

\begin{table*}[t]
	\caption{Detection Performance Comparison against Batch-based Algorithms}
	\label{tbl:comparison_batch}
	\centering
	\begin{tabular}{c|cccc|cccc|c}		
		\hline 
		\multirow{2}{*}{Video} & \multicolumn{4}{c|}{001} & \multicolumn{4}{c|}{002} & \multirow{2}{*}{Avg($F_1$)}\\
		\cline{2-9}
		& Recall &  Precision&  $F_{1}$ score  & Time Per Frame & Recall &  Precision&  $F_{1}$ score  & Time Per Frame & \\ 
		\hline
		
		RPCA & \textbf{94.57\%} &40.65\% & 56.86\% & \textbf{3.16s} & \textbf{90.15\%} &78.06\% & \underline{83.67\%} & \textbf{5.45s} & 70.27\% \\ 
		\hline 
		
		LSD	& \underline{86.80\%} & \textbf{70.79\%}  & \textbf{77.98\%} & 68.48s & \underline{82.19\%} & \textbf{90.87\%}  & \textbf{86.31\%} & 119.50s &  \textbf{82.15\%} \\
		\Xhline{1pt}
		\textbf{O-LSD} & 64.99\% & \underline{63.75\%}  &  \underline{64.36\%} & \underline{6.57s}  & 73.00\% & \underline{90.21\%}  &  80.69\% & \underline{10.75s} &  \underline{72.48\%} \\
		\Xhline{1pt}
	\end{tabular} 
\end{table*}

Besides the comparison against the state-of-the-art online algorithms, we also compare O-LSD with the batch-based algorithms, which are RPCA \cite{lin2010RPCA} and LSD \cite{liu2015LSD}. 
O-LSD achieves slightly descent detection performance with significantly reduced delay in processing on both videos.

Compared with batch-based approaches, O-LSD achieves comparable performance with the batch-based approach RPCA, since it generates less false alarms, as visualized in \Cref{fig:detection_visualization_methods_sq_1}. 
O-LSD fails in matching the detection performance by its batch-based counterpart, LSD, and \Cref{tbl:comparison_batch} presents a drop of about 10\% in $F_{1}$ score by O-LSD.
This little gap in detection performance between LSD- and O-LSD implies that O-LSD works pretty well, although it may not converge to the global optimum of LSD.

In term of the processing time for each frame, O-LSD significantly reduces this metric, compared with LSD.
As presented in \Cref{tbl:comparison_batch}, the time cost per frame for O-LSD is ten times smaller than LSD. 
Compared with RPCA, O-LSD improves the detection performance with moderately increased time cost per frame. 
As the detection results by those batch-based approaches are not available until the entire optimization is completed, O-LSD significantly reduces the delay in moving object detection by the O-LSD.

\subsection{Performance Evaluation with Temporally Down-sampling}

Additionally, we evaluate the effects of temporally down-sampling on detection performance.
In this set of experiments, one frame from every $T$ frames, $T \in \{1, 3, 5, 10 \}$, is fed to O-LSD.

As shown in \Cref{tbl:temporally_downsampling}, with increasing $T$ from 1 to 10, the $F_1$ scores fluctuate negligible, which implies that detection performance of O-LSD is almost not influenced by the temporally down-sampling.
At the same time, the time costs for each frame are more or less the same with different temporally down-sampling frequencies.
Since fewer frames need to be processed with large temporally down-sampling scales, the delay for processing is reduced. 
However, considering that it takes more than 1 second to obtain detection results for each frame by O-LSD, it is still hard to directly apply O-LSD for real-time applications without appropriate accelerations.

\begin{figure*}[t]
	\footnotesize
	\centering
	\begin{tabular}{cccccc}
		\centering
		
	 	\includegraphics[width=0.145 \linewidth]{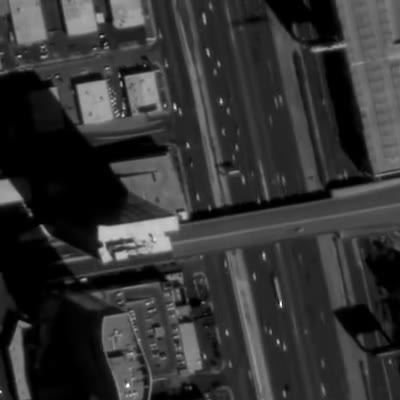}	&	
		\includegraphics[width=0.145 \linewidth]{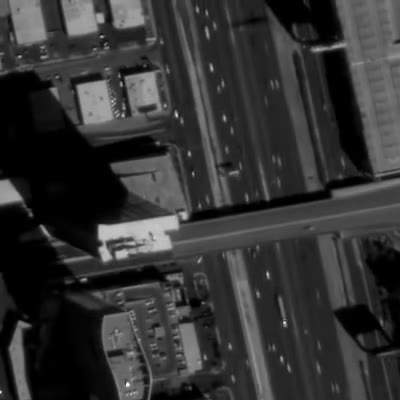}	&	
		\includegraphics[width=0.145 \linewidth]{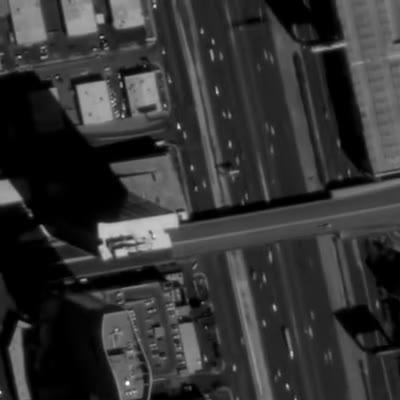}	& 
		\includegraphics[width=0.145 \linewidth]{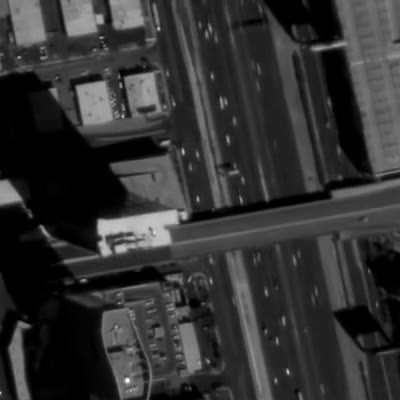}	&	
		\includegraphics[width=0.145 \linewidth]{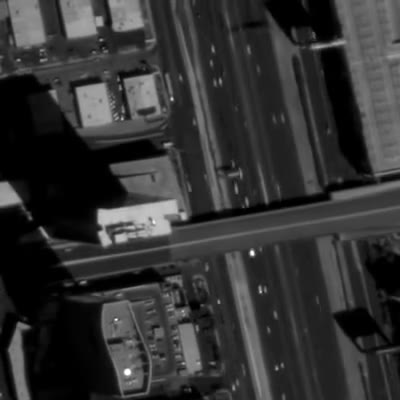} &
		\includegraphics[width=0.145 \linewidth]{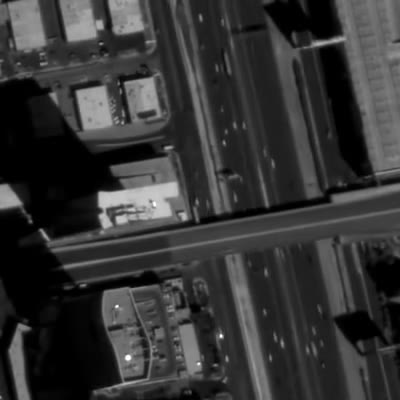} \\
		
		\includegraphics[width=0.145 \linewidth]{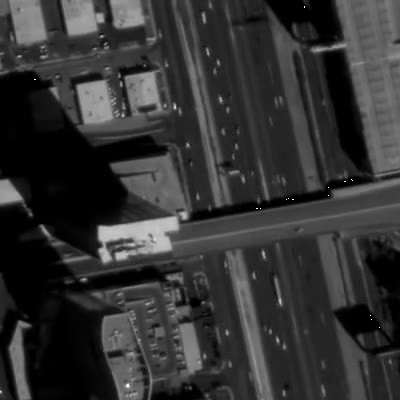}	&	
		\includegraphics[width=0.145 \linewidth]{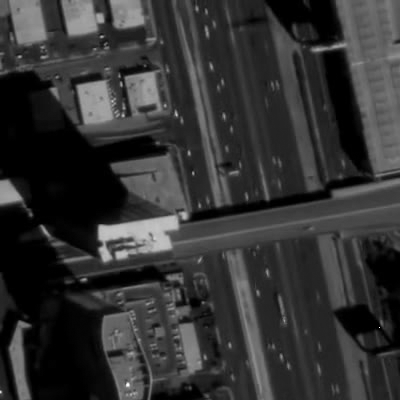}	&	
		\includegraphics[width=0.145 \linewidth]{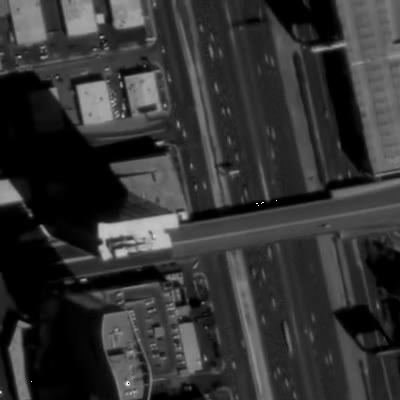}	& 
		\includegraphics[width=0.145 \linewidth]{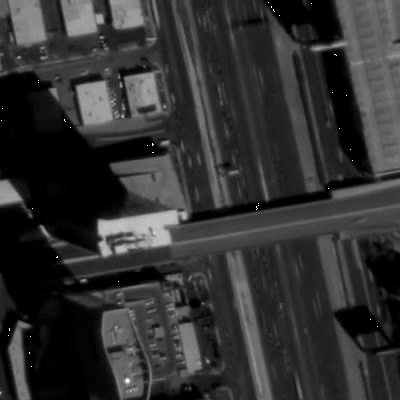}	&	
		\includegraphics[width=0.145 \linewidth]{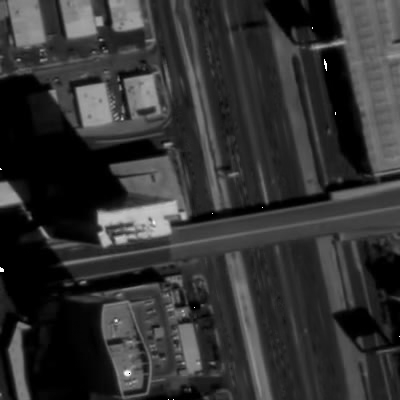} &
		\includegraphics[width=0.145 \linewidth]{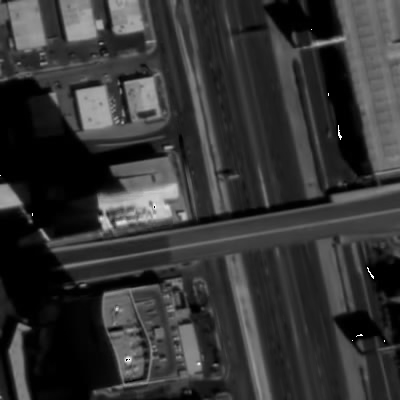} \\
		
		\includegraphics[width=0.145 \linewidth]{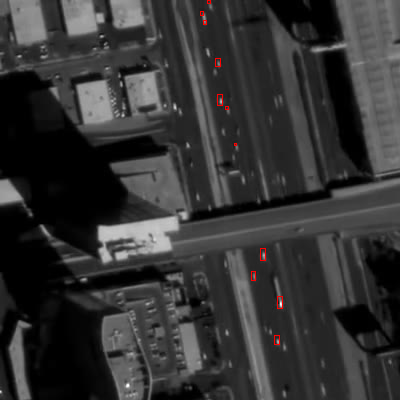}	&	
		\includegraphics[width=0.145 \linewidth]{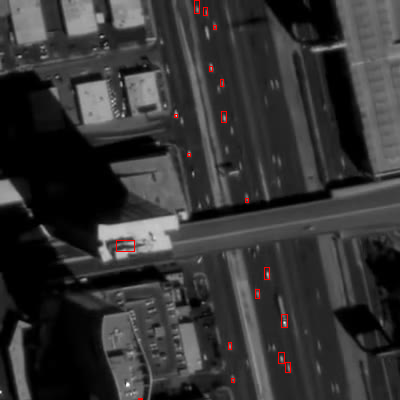}	&	
		\includegraphics[width=0.145 \linewidth]{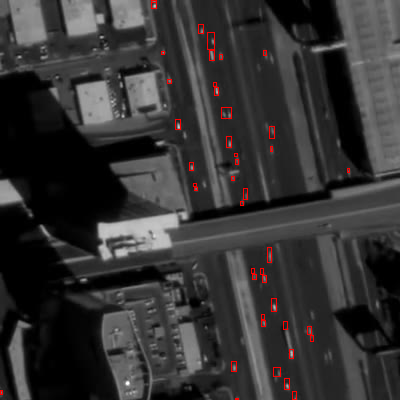}	& 
		\includegraphics[width=0.145 \linewidth]{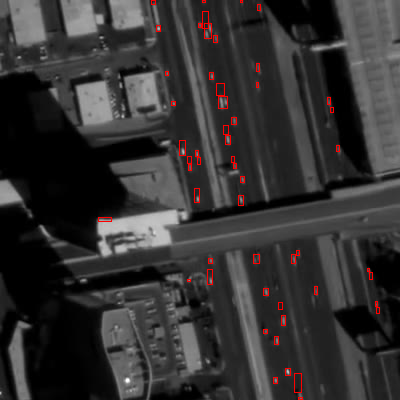}	&	
		\includegraphics[width=0.145 \linewidth]{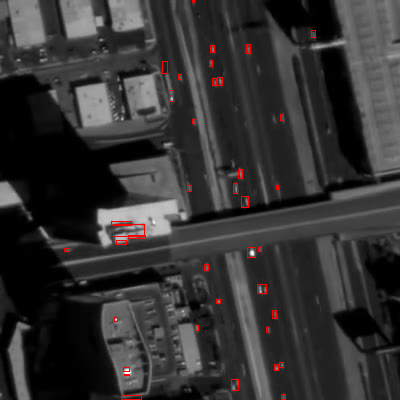} &
		\includegraphics[width=0.145 \linewidth]{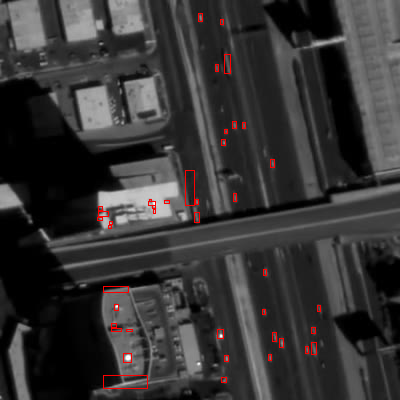} \\
		
		Frame-10	&	Frame-25	&	Frame-50	&	Frame-100	&	Frame-250 & Frame-500	 \\
	\end{tabular}
	\caption{Visualization of the estimated background and detection by O-LSD on Video 001. Images in the 1st row are the original input, and those in the 2nd and 3rd rows are the estimated background frames and detection results, respectively.} 
	\label{fig:detection_visualization_sq_1}
\end{figure*}

\begin{figure*}[t]
	\footnotesize
	\centering
	\begin{tabular}{cccc}
		\centering
		
		\includegraphics[width=0.21 \linewidth]{figs/det/olsd_sq_1/f-249}	&
		\includegraphics[width=0.21 \linewidth]{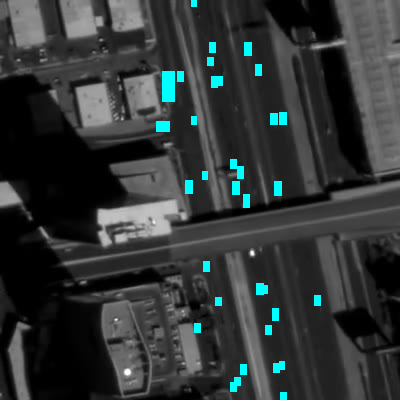}	&
		\includegraphics[width=0.21 \linewidth]{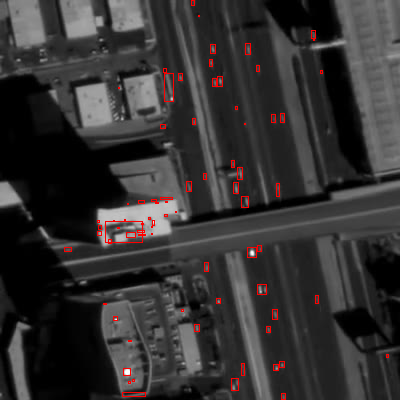}	&	
		\includegraphics[width=0.21 \linewidth]{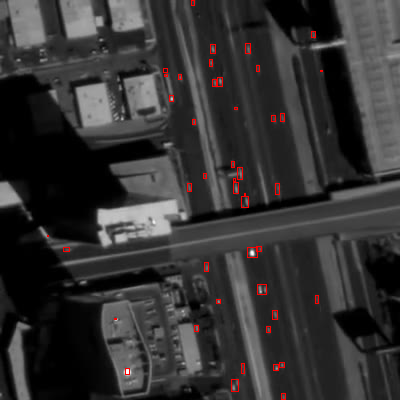}	\\	
		
		Input( Frame-250) & Ground Truth    & RPCA & LSD  \\
			
		\includegraphics[width=0.21 \linewidth]{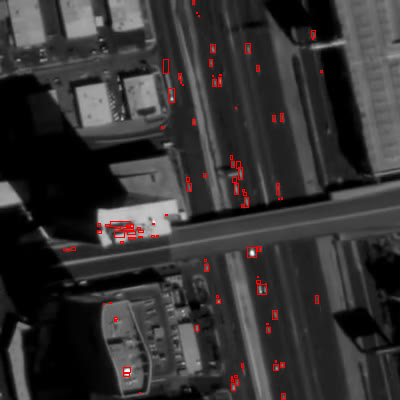}	&	
		\includegraphics[width=0.21 \linewidth]{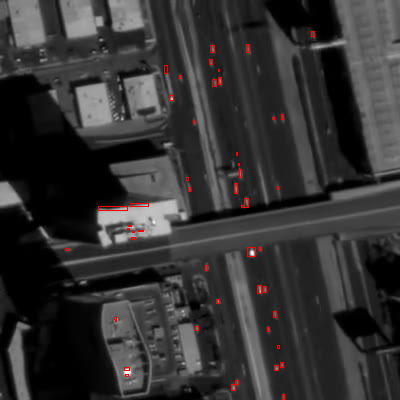}	& 
		\includegraphics[width=0.21 \linewidth]{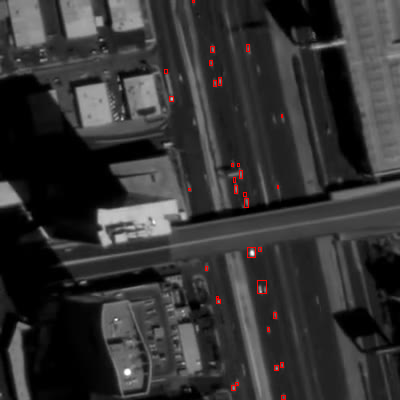}	&
		\includegraphics[width=0.21 \linewidth]{figs/det/olsd_sq_1/det-249}	 \\
		
		GRASTA & OR-PCA & GUSOS & \textbf{O-LSD}  \\

	\end{tabular}
	\caption{Detection results obtained by different algorithms.} 
	\label{fig:detection_visualization_methods_sq_1}
\end{figure*}

\section{Conclusion}
\label{sec:conclusion}

The main contribution of this paper is a effective algorithm, \textbf{O}nline \textbf{L}ow-rank and \textbf{S}tructured Sparse \textbf{D}ecomposition (O-LSD), which combines the stochastic optimization and structured sparsity penalty to improve online subspace estimation method for moving object detection in satellite videos. 
We elaborate the model of O-LSD and its optimization method that is proved to converge almost surely under mild condition.
The experiments on a dataset of two satellite videos validate the improvement of O-LSD to the existing state-of-the-art approaches.
With temporal down-sampling scheme, O-LSD also reduces the processing delay with almost unchanged performance.

\section*{Acknowledgment}

This work is partially supported by China Scholarship Council. The authors would like to thank Planet Team for providing the data in this research \cite{team2016planet}.

\appendices

\section{Technical Lemma}
\label{sec:lemma}

\begin{lem}[Danskin's Theorem from \cite{bertsekas1997nonlinear_optimization}]
	\label{lem:Danskin_theorem_subgradient}
	Let $\mathcal{C} \subset \mathbb{R}^{m}$ be a compact set.
	The function $\ell(\mathbf{x}, \mathbf{u}): \mathbb{R}^{n} \times \mathcal{C} \to \mathcal{R}$ is continuous, and $\ell(\cdot, \mathbf{u})$ is convex with regards to $\mathbf{x}$ for every $u \in \mathcal{C}$.
	Define $\ell(\mathbf{x}) = \min_{\mathbf{u} \in \mathcal{C}} \ell(\mathbf{x}, \mathbf{u}) $ and $\mathcal{C}(\mathbf{x}) = \{\mathbf{u}^{*} | \mathbf{u}^{*} = \argmin_{\mathbf{u}} \ell(\mathbf{x}, \mathbf{u}) \}$.

	If $\ell(\mathbf{x}, \mathbf{u})$ is differentiable with respect to $\mathbf{x}$ for all $\mathbf{u} \in \mathcal{C}$, and $\frac{\partial \ell(\mathbf{x}, \mathbf{u}) }{\partial \mathbf{x}} $ is continuous with respect to $\mathbf{u}$ for all $\mathbf{x}$, then the sub-gradient of $\ell(\mathbf{x})$ is given by 
	\begin{equation}
	\partial_{\mathbf{x}} \ell(\mathbf{x}) = \conv \{
	\frac{\partial \ell(\mathbf{x}, \mathbf{u})} {\partial \mathbf{x}} 
	| \mathbf{u} \in \mathcal{C}(\mathbf{x})
	\} 
	\end{equation}
	where $\conv\{\cdot\}$ indicates the convex hull operator. 
	
\end{lem}

\begin{lem}[Lemma 2.6 from \cite{shalev2012online_convex_optimization}]
	\label{lem:bounded_gradient_Lipschitz}
	Let $f: \mathbf{X} \to \mathbb{R}$ be a convex function.
	Then, $f$ is $L$-Lipschitz over $\mathbf{X}$ with respect to a norm $\left\Vert \cdot \right\Vert$ if and only if for all $\forall \mathbf{w} \in \mathbf{X}$ and $\mathbf{z} \in \partial f(\mathbf{w})$ we have that $\left\Vert \mathbf{z} \right\Vert_{*} \le L$, where $\left\Vert \cdot \right\Vert_{*}$ is the dual norm.
\end{lem}

\begin{lem}[Sufficient condition of convergence for a stochastic optimization from \cite{shalev2012online_convex_optimization}]
	\label{lem:convergence_g_t}
	Let $(\Omega, \mathcal{F}, P)$ be a measurable probability space, $\mu_{t}$, for $t > 0$, be the realization of a stochastic process and $\mathcal{F}_{t}$ be the filtration by the past information at time $t$. 
	Let 
	\begin{equation}
	\begin{aligned}
	\delta_{t} = 
	\begin{cases}
	1, &  \text{if } \mathbb{E}[u_{t+1} - u_{t} | \mathcal{F}_t] > 0, \\
	0, &  \text{otherwise}. 
	\end{cases}
	\end{aligned}
	\end{equation}
	
	If for all $t$, $\mu_t \ge 0$ and $\sum_{t=1}^{\infty} \mathbb{E}[\delta_t(u_{t+1} - u_{t}) ] < \infty $, then $\mu_t$ is a quasi-martingale and converges almost surely. Moreover,
	\begin{equation}
	\sum_{t=1}^{\infty} \left| \mathbb{E}[u_{t+1} - u_{t} | \mathcal{F}_t] \right| < + \infty \text{ a.s. }
	\end{equation}
	
\end{lem}

\begin{lem}[Corollary of Donsker theorem from \cite{bensoussan2011Donsker_theorem}]
	\label{lem:Donsker_theorem}
	Let $\mathbf{F} = \{ f_{\theta}: \mathcal{X} \to \mathbb{R}, \theta \in \Theta \}$ be a set of measurable functions indexed by a bounded subset $\Theta $ of $\mathbb{R}^{d}$. 
	Suppose that there exists a constant $K$ such that 
	\begin{equation}
	\begin{aligned}
	\left| f_{\theta_{1}} - f_{\theta_{2}} \right| \le K \left\Vert \theta_{1} - \theta_{2} \right\Vert_{2},
	\end{aligned}
	\end{equation}
	for every $\theta_{1}$ and $\theta_{2}$ in $\Theta$ and $\mathbf{x}$ in $\mathcal{X}$. 
	Then $F$ is P-Donsker.
	For any $f$ in $F$, let us define $\mathbb{P}_{n} f$, $\mathbb{P} f$ and $\mathbb{G}_{n} f$ as 
	\begin{equation}
	\begin{aligned}
	& \mathbb{P}_{n} f = \frac{1}{n} \sum_{i=1}^{n} f(\mathbf{X}_{i}), \\
	& \mathbb{P} f = \mathbb{E}[f(X)], \\
	& \mathbb{G}_{n} f = \sqrt{n} (\mathbb{P}_{n} f - \mathbb{P} f ).
	\end{aligned}
	\end{equation}
	Let us also suppose that for all $\mathbb{P} f^{2} \le \delta^{2}$ and $\left\Vert f \right\Vert_{\infty} \le M$
	and that the random variables $\mathbf{X}_{1}, \mathbf{X}_{2}, \cdots$ are Borel-measurable. 
	Then, we have
	\begin{equation}
	\mathbb{E} \left| \mathbb{G}_{n} f \right|_{F}= O(1),
	\end{equation}
	where $\left| \mathbb{G}_{n} f \right|_{F}= \sup_{f \in F} \left| \mathbb{G}_{n} f \right|$.
	
\end{lem}

\begin{lem}[Positive converging sums from \cite{mairal2010online_dictionary_learning}]
	\label{lem:positive_converging_sums}
	
	Let $a_{n}$,$b_{n}$ be two real sequences such that for all $n$,$a_{n} \ge 0$,$b_{n} \ge 0$, $\sum_{n=1}^{\infty} a_{n} = \infty$, $\sum_{n=1}^{\infty} a_{n} b_{n} < \infty$, $\exists K > 0 \text{ s.t.} \left| b_{n+1} - b_{n} \right| < K a_{n}$. 
	Then, $\lim_{n \to \infty} b_{n} = 0$.
	
\end{lem}

\section{Proposition}
\label{sec:prop}

\begin{prop}
	\label{prop:uniformly_bounded}
	Assume $\mathbf{d} \in \mathbf{D}$ is uniformly bounded, and $(\mathbf{r}^{*}, \mathbf{s}^{*})$ is the minimizers of the reconstruction cost function $\hat{\ell}(\mathbf{d}, \mathbf{L}, \mathbf{r}, \mathbf{s})$ obtained by \Cref{alg:frame_r_s}. Then,
	\begin{enumerate}
		\item $\mathbf{r}^{*}$ and $\mathbf{s}^{*}$ is uniformly bounded;
		\item $\frac{1}{t} \mathbf{A}_t $ and $\frac{1}{t} \mathbf{B}_t$ is uniformly bounded;
		\item $\mathbf{L}_t$ is supported by a compact subset $\mathcal{L}$,
	\end{enumerate}
\end{prop}

\begin{proof}
	\label{proof:uniformly_bounded}
	
	Given $(\mathbf{0}, \mathbf{d})$ is a non-trivial feasible solution to \Cref{eq:reconstruction_cost_function}, for the optimal solution $(\mathbf{r}^{*}, \mathbf{s}^{*})$,
	\begin{equation}
	\begin{aligned}
	& \frac{1}{2} \left\Vert \mathbf{d}- \mathbf{L} \mathbf{r}^{*} - \mathbf{s}^{*} \right\Vert^{2}_{2}
	+  \frac{\lambda_{1}}{2} \left\Vert \mathbf{r}^{*}  \right\Vert_{2}^{2} 
	+  \lambda_{2} \left\Vert \mathbf{s}^{*} \right\Vert_{\ell_{1}/\ell_{\infty}} \\
	\le & \hat{\ell}(\mathbf{d}, \mathbf{L}, \mathbf{0}, \mathbf{d}) 
	\le \lambda_{2} \left\Vert \mathbf{d} \right\Vert_{\ell_{1}/\ell_{\infty}} ,
	\end{aligned}
	\end{equation}
	thus, we obtain that
	\begin{equation}
	\begin{aligned}
	\left\Vert \mathbf{r}^{*}  \right\Vert_{2}^{2} 
	\le \frac{2 \lambda_{2} }{\lambda_{1}}  \left\Vert \mathbf{d} \right\Vert_{\ell_{1}/\ell_{\infty}} .\\
	\left\Vert \mathbf{s}^{*} \right\Vert_{\ell_{1}/\ell_{\infty}} 
	\le \left\Vert \mathbf{d} \right\Vert_{\ell_{1}/\ell_{\infty}} .
	\end{aligned}
	\end{equation}
	Based on the assumption that $\mathbf{d}$ is uniformly bounded, then $\mathbf{r}^{*}, \mathbf{s}^{*}$ is uniformly bounded.
	
	Similarly, we show that the accumulation matrices $\mathbf{A}_t$ and $\mathbf{B}_t$ are also uniformly bounded, as
	\begin{equation}
	\begin{aligned}
	\frac{1}{t} \mathbf{A}_t & = \frac{1}{t} \sum_{i=1}^{t} \mathbf{r}_{i} \mathbf{r}_{i}^{T}, \\
	\frac{1}{t} \mathbf{B}_t & = \frac{1}{t} \sum_{i=1}^{t} (\mathbf{d}_{i} - \mathbf{s}_{i}) {\mathbf{r}_{i}}^T.
	\end{aligned}
	\end{equation}
	
	The closed-from solution $\mathbf{L}_t$ is given as
	\begin{equation}
	\begin{aligned}
	\mathbf{L}_t & = \mathbf{B}_t (\mathbf{A}_t + \lambda_{1} \mathbf{I} )^{-1} \\
	& = \frac{1}{t} \mathbf{B}_t ( \frac{1}{t} \mathbf{A}_t + \frac{\lambda_{1} }{t} \mathbf{I} )^{-1}.
	\end{aligned}
	\end{equation}
	in which $\frac{1}{t} \mathbf{A}_t$ and $\frac{1}{t} \mathbf{B}_t$ is uniformly bounded, therefore, $\mathbf{L}_t $ is uniformly bounded.
	
\end{proof}

\begin{prop}
	\label{prop:surrogate_function_Lispchitz}
	Let $\mathbf{r}$, $\mathbf{s}$, $\mathbf{L}_{t}$ be the solution obtained by \Cref{alg:online_alg},
	
	\begin{enumerate}
		\item $\hat{\ell}(\mathbf{d}, \mathbf{L}, \mathbf{r}, \mathbf{s})$ and $ \ell(\mathbf{d}, \mathbf{L})$ are uniformly bounded;
		
		\item The surrogate function $g_{t}(\mathbf{L})$ is uniformly bounded and Lipschitz.
	\end{enumerate}
\end{prop}

\begin{proof}
	\label{proof:surrogate_function_Lispchitz}
	The first claim is proved by combining the definition of $\hat{\ell}(\mathbf{d}, \mathbf{L}, \mathbf{r}, \mathbf{s})$ and the uniform boundedness of $\mathbf{d}$, $\mathbf{L}$, $\mathbf{r}$ and $\mathbf{s}$..
	Similarly, we can show $g_{t}(\mathbf{L})$ is uniformly bounded.
	
	To proof $g_{t}(\mathbf{L})$ is Lipschitz, we show that the gradient of $g_{t}(\mathbf{L})$ is uniformly bounded as 
	\begin{equation}
	\begin{aligned}
	\left\Vert \nabla g_{t}(\mathbf{L}) \right\Vert_{F} = &
	\left\Vert \mathbf{L} (\frac{1}{t} \mathbf{A}_t - \frac{1}{t}\mathbf{I}) - \frac{1}{t} \mathbf{B}_t \right\Vert_{F} \\
	\le &
	\left\Vert \mathbf{L} \right\Vert_{F} \left(\left\Vert \frac{1}{t} \mathbf{A}_t \right\Vert_{F} + \left\Vert \frac{1}{t} \mathbf{I} \right\Vert_{F}\right)  + \left\Vert \frac{1}{t} \mathbf{B}_t \right\Vert_{F}
	\end{aligned}
	\end{equation}
	where the terms on the right side of the inequality are uniformly bounded, $\left\Vert \nabla g_{t}(\mathbf{L}) \right\Vert_{F}$ is uniformly bounded. 
	According to \Cref{lem:bounded_gradient_Lipschitz}, $g_{t}(\mathbf{L})$ is convex with respect to $\mathbf{L}$, the boundedness of the gradient implies that $g_{t}(\mathbf{L})$ is Lipschitz.
	
\end{proof}

\begin{prop}
	\label{prop:reconstruction_cost_function_Lipschitz}
	$\mathcal{X}_{0}(\mathbf{L})$ refers to the set of all minimizers to $\hat{\ell}(\mathbf{d}, \mathbf{L}, \mathbf{r}, \mathbf{s})$ as
	\begin{equation}
	\begin{aligned}
	\mathcal{X}_{0}(\mathbf{L}) = \{ 
	(\mathbf{\overline{r}}, \mathbf{\overline{s}}) |
	(\mathbf{\overline{r}}, \mathbf{\overline{s}}) = \argmin_{\mathbf{r}, \mathbf{s}} 
	\hat{\ell}(\mathbf{d}, \mathbf{L}, \mathbf{r}, \mathbf{s}) \}.
	\end{aligned}
	\end{equation}
	
	\begin{enumerate}
		\item The sub-gradient of function $\ell(\mathbf{d}, \mathbf{L})$ with respect to $\mathbf{L}$ is given as 
		\begin{equation}
		\begin{aligned}
		\partial_{L} \ell(\mathbf{d}, \mathbf{L}) =
		\conv\{
		(\mathbf{L} \mathbf{r}^{*} + \mathbf{s}^{*} - \mathbf{d}) { \mathbf{r}^{*}}^{T} \\
		| (\mathbf{r}^{*}, \mathbf{s}^{*}) & \in \mathcal{X}_{0}(\mathbf{L})
		\},
		\end{aligned}
		\end{equation}
		where $\conv\{\cdot \}$ is convex hull operator.
		
		\item The subgradient $\partial_{L} \ell(\mathbf{d}, \mathbf{L})$ is uniformly bounded, and $\ell(\mathbf{d}, \mathbf{L})$ is uniformly Lipschitz.
	\end{enumerate}
	
\end{prop}

\begin{proof}
	\label{proof:reconstruction_cost_function_Lipschitz}
	
	To the best of our knowledge, there is no available necessary and sufficient condition for uniqueness of minimizer to the reconstruction cost function, which implies that more than one minimizers of $\hat{\ell}(\mathbf{d}, \mathbf{L}, \mathbf{r}, \mathbf{s})$ probably exist. 
	
	$\hat{\ell}(\mathbf{z}, \mathbf{L}, \mathbf{r}, \mathbf{s})$ is convex and differentiable with respect to $\mathbf{L}$ for every feasible $(\mathbf{r}^{*}, \mathbf{s}^{*})$.
	Given $\frac{\partial \hat{\ell}(\mathbf{z}, \mathbf{L}, \mathbf{r}, \mathbf{s})}{\partial \mathbf{L}} = (\mathbf{L} \mathbf{r} + \mathbf{s} - \mathbf{d}) { \mathbf{r}}^{T}$ is differentiable with respect to $(\mathbf{r}, \mathbf{s})$ for all $\mathbf{L}$, according to \Cref{lem:Danskin_theorem_subgradient}, the sub-gradient of $\ell(\mathbf{d}, \mathbf{L})$ is given as 
	\begin{equation}
	\begin{aligned}
	\partial_{L} \ell(\mathbf{d}, \mathbf{L}) \in
	\conv\{
	(\mathbf{L} \mathbf{r}^{*} + \mathbf{s}^{*} - \mathbf{d}) { \mathbf{r}^{*}}^{T} \\
	| (\mathbf{r}^{*}, \mathbf{s}^{*}) & \in \mathcal{Z}_{0}(\mathbf{L})
	\}.
	\end{aligned}
	\end{equation}
	
	Then we proof that $(\mathbf{L} \mathbf{r} + \mathbf{s} - \mathbf{d}) { \mathbf{r}}^{T}$ is uniformly bounded
	\begin{equation}
	\begin{aligned}
	\left\Vert (\mathbf{L} \mathbf{r} + \mathbf{s} - \mathbf{d}) { \mathbf{r}}^{T} \right\Vert_{2} 
	\le  \left\Vert \mathbf{r} \right\Vert_{2} ( 
	\left\Vert \mathbf{L} \right\Vert_{F} \left\Vert \mathbf{r} \right\Vert_{2} 
	+ \left\Vert \mathbf{s} \right\Vert_{2} + \left\Vert \mathbf{d} \right\Vert_{2}),
	\end{aligned}
	\end{equation}
	in which every item on the right side of the inequality is uniformly bounded.
	Therefore $(\mathbf{L} \mathbf{r} + \mathbf{s} - \mathbf{d}) { \mathbf{r}}^{T}$ is uniformly bounded. 
	The convex hull of the bounded set is also bounded, thus the sub-gradient $\partial_{L} \ell(\mathbf{d}, \mathbf{L})$ is bounded.
	
	By \Cref{lem:bounded_gradient_Lipschitz}, $\ell(\mathbf{d}, \mathbf{L})$ is convex with respect to $\mathbf{L}$, and its sub-gradient is is uniformly bounded, thus $\ell(\mathbf{d}, \mathbf{L})$ is uniformly Lipschitz.
	
\end{proof}

\begin{prop}
	\label{prop:empirical_cost_function_Lipschitz}
	
	The empirical cost function $f_t(\mathbf{L})$ is uniformly bounded and Lipschitz.
\end{prop}

\begin{proof}
	\label{proof:empirical_cost_function_Lipschitz}
	By checking the definition of the empirical cost function $f_{n}(\mathbf{L})$, $f_{n}(\mathbf{L})$ is also uniformly bounded.
	
	For $\mathbf{g}_{i} \in \partial_{L} \ell(\mathbf{d}_{i}, \mathbf{L})$,the sub-gradient $\left\Vert \partial_{\mathbf{L}} \right\Vert_{F} $ is uniformly bounded, 
	\begin{equation}
	\begin{aligned}
	\left\Vert \partial_{\mathbf{L}} \right\Vert_{F} 
	= \left\Vert \frac{1}{n} \sum_{i=1}^{n} \mathbf{g}_{i} + \frac{\lambda_{1}}{2n} \mathbf{L} \right\Vert_{F}
	\le \frac{1}{n} \sum_{i=1}^{n} \left\Vert \mathbf{g}_{i} \right\Vert_{F} 
	+ \frac{\lambda_{1}}{2n} \left\Vert \mathbf{L} \right\Vert_{F}.
	\end{aligned}
	\end{equation}
	Because $f_{n}(\mathbf{L})$ is convex with respect to $\mathbf{L}$, and its sub-gradient $\left\Vert \partial_{\mathbf{L}} \right\Vert_{F} $ is uniformly bounded, $f_{n}(\mathbf{L})$ is Lipschitz.
	
\end{proof}

\section{Proof Details}
\label{sec:proof}

\begin{thm}
	\label{thm:surrogate_function_convergence}
	
	Let $\{\mathbf{L}_{t} \}_{t=1}^{\infty}$ be the sequence of solution obtained by \Cref{alg:online_alg}, the surrogate function $g_{t}(\mathbf{L_t})$ converges almost surely.
	
\end{thm}

\begin{proof}
	\label{proof:surrogate_function_convergence}
	
	In stochastic optimization, the expected cost function is defined all the samples,
	\begin{equation}
	f(\mathbf{L}) = \mathbb{E}_{\mathbf{d}} [ \ell(\mathbf{d}, \mathbf{L}) ] 
	= \lim_{n \to \infty} f_{n}(\mathbf{L}).
	\end{equation}
	$g_{t}(\mathbf{L}_t)$ is analyzed as a stochastic positive process, since each term in it is non-negative and samples are drawn randomly (independent).
	
	Note $u_{t} = g_{t}(\mathbf{L}_t)$, the difference between two consecutive time instances is given as
	\begin{equation}
	\label{eq:u_t_difference}
	\begin{aligned}
	& u_{t+1} - u_{t} \\
	= & g_{t+1}(\mathbf{L}_{t+1}) - g_{t}(\mathbf{L}_{t}) \\
	= & g_{t+1}(\mathbf{L}_{t+1}) - g_{t+1}(\mathbf{L}_{t}) + g_{t+1}(\mathbf{L}_{t})- g_{t}(\mathbf{L}_{t}) \\
	= & g_{t+1}(\mathbf{L}_{t+1}) - g_{t+1}(\mathbf{L}_{t}) \\
	& + \frac{\ell(\mathbf{d}_{t+1}, \mathbf{L}_{t}) - f_{t}(\mathbf{L}_{t})}{t+1}
	+ \frac{f_{t}(\mathbf{L}_{t}) - g_{t}(\mathbf{L}_{t}) }{t+1}.
	\end{aligned}
	\end{equation}
	The first two term satisfy $g_{t+1}(\mathbf{L}_{t+1}) \le g_{t+1}(\mathbf{L}_{t})$, and $f_t(\mathbf{L}_{t}) - g_{t}(\mathbf{L}_{t}) \le 0$, therefore,
	\begin{equation}
	\begin{aligned}
	u_{t+1} - u_{t} 
	\le & \frac{\ell(\mathbf{d}_{t+1}, \mathbf{L}_{t}) - f_{t}(\mathbf{L}_{t})}{t+1} \\
	\le & \frac{\ell(\mathbf{d}_{t+1}, \mathbf{L}_{t}) - \frac{1}{t} \sum_{i=1}^{t}  \ell(\mathbf{d}_{t}, \mathbf{L}_{t}) }{t+1} 
	\end{aligned}
	\end{equation}
	The expectation conditioned on past information $\mathcal{F}_t$ is given as
	\begin{equation}
	\begin{aligned}
	\mathbb{E}[ u_{t+1} - u_{t} | \mathcal{F}_{t} ]
	\le & \frac{\mathbb{E}[\ell(\mathbf{d}_{t+1}, \mathbf{L}_{t})| \mathcal{F}_{t}] - \frac{1}{t} \sum_{i=1}^{t}  \ell(\mathbf{d}_{t}, \mathbf{L}_{t}) }{t+1} \\
	\le & \frac{ f(\mathbf{L}_t) - \frac{1}{t} \sum_{i=1}^{t}  \ell(\mathbf{d}_{t}, \mathbf{L}_{t})  }{t+1} \\
	\le & \frac{ \left\Vert f - f_{t} \right\Vert_{\infty}}{t+1},
	\end{aligned}
	\end{equation}
	where we define $f_t = \frac{1}{t} \sum_{i=1}^{t}  \ell(\mathbf{d}_{t}, \mathbf{L}_{t})$, $f = \mathbb{E}_{\mathbf{d}} [\ell(\mathbf{d}, \mathbf{L}_{t}) ]$, and $\left\Vert f - f_{t} \right\Vert_{\infty} = \sup_{f \in \mathcal{F}} \left| f - f_{t}\right|$.
	$F= \{ \ell(\mathbf{d}, \mathbf{L}): \mathcal{D} \to \mathbb{R}, \mathbf{L} \in \mathcal{L} \} $ defines a set of measurement function indexed by $\mathbf{L}_{t}$ from a compact subset $\mathcal{L}$, and P-Donsker.
	In addition, the boundedness of $\ell(\mathbf{d}, \mathbf{L})$ implies that $\mathbb{E} [\ell(\mathbf{d}, \mathbf{L})^{2}]$ is uniformly bounded.
	Then, the requirements of \Cref{lem:Donsker_theorem} are all satisfied such that 
	\begin{equation}
	\begin{aligned}
	\mathbb{E}[\sqrt{t}\left\Vert f - f_{t} \right\Vert_{\infty} ] \le \kappa, \kappa > 0.
	\end{aligned}
	\end{equation}
	Therefore,
	\begin{equation}
	\begin{aligned}
	\mathbb{E}[ [\mathbb{E}[ u_{t+1} - u_{t} | \mathcal{F}_{t} ] ]^{+} ]  
	= & \mathbb{E}[ \max \{0, \mathbb{E}[ u_{t+1} - u_{t} | \mathcal{F}_{t} ]\} ] \\
	\le & \frac{\kappa}{t^{\frac{3}{2}}},
	\end{aligned}
	\end{equation}
	where $[\cdot]^{+}$ is the positive variation operator.
	
	Then we deploy \Cref{lem:convergence_g_t} to present the convergence of $g_t(\mathbf{x})$.
	We define that
	\begin{equation}
	\begin{aligned}
	\delta_{t} =
	\begin{cases}
	1, & \text{if } \mathbb{E}[u_{t+1} - u_{t} | \mathcal{F}_t] > 0 \\
	0, & \text{otherwise}
	\end{cases}
	\end{aligned}.
	\end{equation}
	We have
	\begin{equation}
	\begin{aligned}
	\sum_{t=1}^{\infty} \mathbb{E}[\delta_{t} ( u_{t+1} - u_{t} ) | \mathcal{F}_{t}  ]
	= &\sum_{t=1}^{\infty} \mathbb{E}[ [ \mathbb{E}[ u_{t+1} - u_{t} | \mathcal{F}_{t} ] ]^{+} ] \\
	\le & \sum_{t=1}^{\infty} \frac{\kappa}{t^{\frac{3}{2}}} 
	<  + \infty .
	\end{aligned}
	\end{equation}
	Conclusively, $g_t(\mathbf{L}_t)$ is quasi-martingale and converges almost sure. Moreover,
	\begin{equation}
	\begin{aligned}
	\sum_{t=1}^{\infty} \left|  \mathbb{E}[ u_{t+1} - u_{t} | \mathcal{F}_{t} ] \right| < +\infty \text{ a.s.}
	\end{aligned}
	\end{equation}
	
\end{proof}

\begin{thm}
	\label{thm:solution_convergence}
	
	For two solutions produced by \Cref{alg:online_alg} at two consecutive time instances, 
	\begin{equation}
	\left\Vert \mathbf{L}_{t} - \mathbf{L}_{t+1} \right\Vert_{F} = O(\frac{1}{t}).
	\end{equation}
	
\end{thm}

\begin{proof}
	\label{proof:solution_convergence}
	
	The Hessian matrix of $g_t(\mathbf{L})$ is $\mathbf{H} = \mathbf{I} \otimes (\mathbf{A_t} + \lambda_{1} \mathbf{I})$, where $\otimes$ is the Kronecker production operator. 
	By the definition of $\mathbf{A}_{t}$, $\mathbf{A_t}$ is always semi-positive, therefore, the smallest eigenvalue of $\mathbf{H}$ is greater than $\lambda_{1}$, which implies that $g_t(\mathbf{L})$ is strictly convex (perhaps strongly convex) with respect to $\mathbf{L}$ and 
	\begin{equation}
	\label{eq:L_t_left_hand}
	\begin{aligned}
	g_t(\mathbf{L}_{t+1}) - g_t(\mathbf{L}_{t}) 
	\ge \lambda_{1} \left\Vert \mathbf{L}_{t+1} - \mathbf{L}_{t} \right\Vert_{F}^{2}
	\end{aligned}
	\end{equation}
	
	Because $\mathbf{L}_{t+1}$ minimizes $g_{t+1}(\mathbf{L})$, $g_{t+1}(\mathbf{L}_{t+1}) - g_{t+1}(\mathbf{L}_{t})< 0$, which presents that 
	\begin{equation}
	\label{eq:L_t_mid}
	\begin{aligned}
	& g_{t}(\mathbf{L}_{t+1}) - g_{t}(\mathbf{L}_{t}) \\
	= & g_{t}(\mathbf{L}_{t+1}) - g_{t+1}(\mathbf{L}_{t+1}) 
	+ g_{t+1}(\mathbf{L}_{t+1}) - g_{t}(\mathbf{L}_{t}) \\
	\le & g_{t}(\mathbf{L}_{t+1}) - g_{t+1}(\mathbf{L}_{t+1}) 
	+ g_{t+1}(\mathbf{L}_{t}) - g_{t}(\mathbf{L}_{t}) \\
	\le & (g_{t}(\mathbf{L}_{t+1}) - g_{t+1}(\mathbf{L}_{t+1}))
	-  (g_{t}(\mathbf{L}_{t}) - g_{t+1}(\mathbf{L}_{t}) )
	\end{aligned}
	\end{equation}
	
	We define $G_t(\mathbf{L}) = g_{t}(\mathbf{L}) - g_{t+1}(\mathbf{L})$, then gradient of  $G_t(\mathbf{L})$ is extracted as
	\begin{equation}
	\begin{aligned}
	\nabla G_t(\mathbf{L}) 
	= & \nabla g_{t}(\mathbf{L}) - \nabla g_{t+1}(\mathbf{L}) \\
	= & \frac{1}{t} \mathbf{L} \hat{\mathbf{A}}_{t} + \frac{1}{t} \mathbf{B}_{t} -
	\frac{1}{t+1} \mathbf{L} \hat{\mathbf{A}}_{t+1} - \frac{1}{t+1} \mathbf{B}_{t+1} \\
	= & \frac{1}{t} \mathbf{L} (\hat{\mathbf{A}}_{t} - \frac{t}{t+1} \hat{\mathbf{A}}_{t+1} ) 
	+ \frac{1}{t} (\hat{\mathbf{B}}_{t} - \frac{t}{t+1} \hat{\mathbf{A}}_{t+1} ) ,
	\end{aligned}
	\end{equation}
	in which $\hat{\mathbf{A}}_{t} = \mathbf{A}_{t} + \lambda_{1} \mathbf{I}$.
	
	Given that $\mathbf{A}_t$, $\mathbf{B}_t$ and $\mathbf{L}$ is uniformly bounded, the gradient $\nabla G_t(\mathbf{L}) $ is also uniformly bounded, 
	\begin{equation}
	\begin{aligned}
	& \left\Vert \nabla G_t(\mathbf{L}) \right\Vert_{F} \\
	\le & \frac{1}{t} (
	\left\Vert \mathbf{L} \right\Vert_{F} \left\Vert \hat{\mathbf{A}}_{t} - \frac{t}{t+1} \hat{\mathbf{A}}_{t+1} \right\Vert_{F}
	+ \left\Vert \hat{\mathbf{B}}_{t} - \frac{t}{t+1} \hat{\mathbf{A}}_{t+1}  \right\Vert_{F}
	).
	\end{aligned}
	\end{equation}
	Thus, we conclude that $G_t(\mathbf{L})$ is uniformly Lipschitz with respect to $\mathbf{L}$, and there exist a positive constant $\kappa_{t} = \frac{1}{t} ( \left\Vert \mathbf{L} \right\Vert_{F} \left\Vert \hat{\mathbf{A}}_{t} - \frac{t}{t+1} \hat{\mathbf{A}}_{t+1} \right\Vert_{F}
	+ \left\Vert \hat{\mathbf{B}}_{t} - \frac{t}{t+1} \hat{\mathbf{A}}_{t+1}  \right\Vert_{F})$ that satisfies 
	\begin{equation}
	\label{eq:L_t_right_hand}
	\begin{aligned}
	G_t(\mathbf{L}_{t}) - G_t(\mathbf{L}_{t+1}) 
	\le \kappa_{t} \left\Vert \mathbf{L}_{t+1} - \mathbf{L}_{t}\right\Vert_{F}.
	\end{aligned}
	\end{equation}
	Combining \Cref{eq:L_t_left_hand,eq:L_t_mid,eq:L_t_right_hand}, we conclude that 
	\begin{equation}
	\begin{aligned}
	\left\Vert \mathbf{L}_{t+1} - \mathbf{L}_{t} \right\Vert_{F} 
	\le \frac{\kappa_{t}}{\lambda_{1}}.
	\end{aligned}
	\end{equation}
	and $\left\Vert \mathbf{L}_{t+1} - \mathbf{L}_{t} \right\Vert_{F} = O(\frac{1}{t}) $.
	
\end{proof}

\begin{thm}
	\label{thm:surrogate_function_converge_to_emiprical_cost_function}
	
	Note $f_{t}(\mathbf{L})$ is the empirical cost function, and $g_{t}(\mathbf{L})$ is its surrogate function. $\mathbf{L}_{t}$ is the solution obtained by \Cref{alg:online_alg}, when $t$ tends to infinity, $g_t(\mathbf{L}_{t}) - f_t(\mathbf{L}_{t})$ converges to 0 almost surely.
\end{thm}

\begin{proof}
	\label{proof:surrogate_function_converge_to_emiprical_cost_function}
	
	This proof is originally presented by \cite{feng2013ORPCA_SO}, for the completeness of this proof, we introduce it here. Combining \Cref{eq:u_t_difference} with $g_{t+1}(\mathbf{L}_{t+1}) - g_{t+1}(\mathbf{L}_{t}) \le 0$, we obtain
	\begin{equation}
	\begin{aligned}
	& \frac{g_{t}(\mathbf{L}_{t}) - f_{t}(\mathbf{L}_{t}) }{t+1} \\
	\le & \frac{\ell(\mathbf{d}_{t+1}, \mathbf{L}_{t}) - f_{t}(\mathbf{L}_{t})}{t+1}
	- (g_{t+1}(\mathbf{L}_{t+1}) - g_{t}(\mathbf{L}_{t}) ) \\
	\le & \frac{\ell(\mathbf{d}_{t+1}, \mathbf{L}_{t}) - f_{t}(\mathbf{L}_{t})}{t+1} 
	+ [g_{t+1}(\mathbf{L}_{t+1}) - g_{t}(\mathbf{L}_{t})]^{-},
	\end{aligned}
	\end{equation}
	in which $[\cdot]^{-}$ refers to the negative variation operator.
	
	Similar to the proof of \Cref{thm:surrogate_function_convergence}, we take the expectation conditioned on past information $\mathcal{F}_{t}$
	\begin{equation}
	\begin{aligned}
	\mathbb{E} [ \frac{g_{t}(\mathbf{L}_{t}) - f_{t}(\mathbf{L}_{t}) }{t+1} | \mathcal{F}_t  ] 
	= & \frac{g_{t}(\mathbf{L}_{t}) - f_{t}(\mathbf{L}_{t}) }{t+1} \\
	\le & \mathbb{E} [ \frac{\ell(\mathbf{d}_{t+1}, \mathbf{L}_{t}) - f_{t}(\mathbf{L}_{t})}{t+1} | \mathcal{F}_t ] \\
	& + \mathbb{E} [[g_{t+1}(\mathbf{L}_{t+1}) - g_{t}(\mathbf{L}_{t})]^{-} | \mathcal{F}_t ].
	\end{aligned}
	\end{equation}
	Accumulating $\frac{g_{t}(\mathbf{L}_{t}) - f_{t}(\mathbf{L}_{t}) }{t+1} $ with $t$ tending to $\infty$, we have
	\begin{equation}
	\begin{aligned}
	& \sum_{t= 1}^{\infty} \frac{g_{t}(\mathbf{L}_{t}) - f_{t}(\mathbf{L}_{t}) }{t+1} \\
	\le & \sum_{t= 1}^{\infty} \mathbb{E} [ \frac{\ell(\mathbf{d}_{t+1}, \mathbf{L}_{t}) -  \frac{1}{t} \sum_{i=1}^{t}  \ell(\mathbf{d}_{t}, \mathbf{L}_{t}) }{t+1} | \mathcal{F}_t ] \\
	& + \sum_{t= 1}^{\infty} \mathbb{E} [[g_{t+1}(\mathbf{L}_{t+1}) - g_{t}(\mathbf{L}_{t})]^{-} | \mathcal{F}_t ]\\
	\le & \sum_{t= 1}^{\infty}  \frac{ \vert f - f_{t} \vert }{t+1}  + \sum_{t= 1}^{\infty} \mathbb{E} [[g_{t+1}(\mathbf{L}_{t+1}) - g_{t}(\mathbf{L}_{t})]^{-} | \mathcal{F}_t ].
	\end{aligned}
	\end{equation}
	
	According to central limit theorem, $\sqrt{t}\vert f - f_{t} \vert $ converges almost surely as $t$ tends to infinity.
	From \Cref{thm:surrogate_function_convergence}, we also obtain that 
	\begin{equation}
	\begin{aligned}
	\sum_{t= 1}^{\infty} \vert \mathbb{E} [[g_{t+1}(\mathbf{L}_{t+1}) - g_{t}(\mathbf{L}_{t})]^{-} | \mathcal{F}_t ] \vert < + \infty.
	\end{aligned}
	\end{equation}
	Hence, we have the almost sure convergence of the positive sum
	\begin{equation}
	\begin{aligned}
	& \sum_{t= 1}^{\infty} \frac{g_{t}(\mathbf{L}_{t}) - f_{t}(\mathbf{L}_{t}) }{t+1}  \\
	\le & \sum_{t= 1}^{\infty}  \frac{ \vert f - f_{t} \vert }{t+1}  + \sum_{t= 1}^{\infty} \mathbb{E} [[g_{t+1}(\mathbf{L}_{t+1}) - g_{t}(\mathbf{L}_{t})]^{-} | \mathcal{F}_t ] \\
	< & \infty
	\end{aligned}
	\end{equation}

	As demonstrated in \Cref{prop:surrogate_function_Lispchitz,prop:empirical_cost_function_Lipschitz}, $g_t(\mathbf{L})$ and $f_t(\mathbf{L})$ are both Lipschitz with respect to $\mathbf{L}$, which implies there exists $\kappa_{2}$ such that 
	\begin{equation}
	\begin{aligned}
	& \left| g_{t+1}(\mathbf{L}_{t+1})- f_{t+1}(\mathbf{L}_{t+1}) 
	- ( g_t(\mathbf{L}_t) - f_t(\mathbf{L}_t) ) \right|  \\
	= & \left| g_{t+1}(\mathbf{L}_{t+1})- g_t(\mathbf{L}_t) 
	- (f_{t+1}(\mathbf{L}_{t+1})- f_t(\mathbf{L}_t) ) \right| \\
	\le & \left| g_{t+1}(\mathbf{L}_{t+1})- g_t(\mathbf{L}_t) \right| 
	+ \left| f_{t+1}(\mathbf{L}_{t+1})- f_t(\mathbf{L}_t) \right| \\
	\le & \left| g_{t+1}(\mathbf{L}_{t+1})- g_{t+1}(\mathbf{L}_t) \right|
	+  \left| g_{t+1}(\mathbf{L}_{t})- g_t(\mathbf{L}_t) \right| \\
	& + \left| f_{t+1}(\mathbf{L}_{t+1})- f_{t+1}(\mathbf{L}_t) \right|
	+  \left| f_{t+1}(\mathbf{L}_{t})- f_t(\mathbf{L}_t) \right| \\ 
	\le & \kappa_{2} \left\Vert \mathbf{L}_{t+1} - \mathbf{L}_{t} \right\Vert_{F}.
	\end{aligned}
	\end{equation}
	
	According to \Cref{lem:positive_converging_sums}, combining with  $\sum_{t=1}^{\infty} \frac{1}{t+1} = \infty$, we have that 
	\begin{equation}
	\begin{aligned}
	\lim_{t \to \infty} g_{t}(\mathbf{L}_{t}) - f_{t}(\mathbf{L}_{t}) = 0, a.s. .
	\end{aligned}
	\end{equation}
	
\end{proof}

\begin{thm}
	\label{thm:surrogate_function_converge_to_expected_cost_function}
	As $t$ tends to infinity, given the $\mathbf{L}_{t}$ is obtained by \Cref{alg:online_alg}, $g_{t}(\mathbf{L}_{t}) - f(\mathbf{L}_{t})$ converges to 0 almost surely.
\end{thm}

\begin{proof}
	\label{proof:surrogate_function_converge_to_expected_cost_function}
	
	From \Cref{thm:solution_convergence}, we have $\left\Vert f - f_{t} \right\Vert_{\infty}$ converges to 0 almost surely, which implies 
	\begin{equation}
	\lim_{t \to \infty} g_t(\mathbf{L}_t) - f(\mathbf{L}_t) = 0, a.s. .
	\end{equation}
	
\end{proof}

\bibliographystyle{IEEEtran}
\footnotesize\bibliography{refs}

\begin{IEEEbiography}[{\includegraphics[width=0.95 \textwidth]{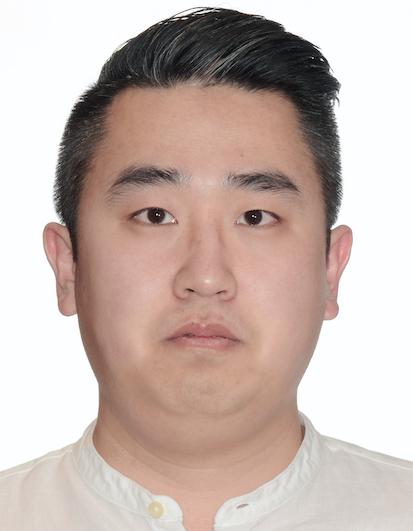}}]{Junpeng Zhang}
	received the B.Sci. degree from the China University of Mining and Technology, Xuzhou, China, in 2013 and the Master's degree in surveying engineering from the same university in 2016. 
	He is currently pursuing the Ph.D. degree in electrical engineering from The University of New South Wales, Australia.
	His research interests include object detection and tracking in remote sensing imaginary.
	He was the winner of "DSTG Best Contribution to Science Award" in Digital Image Computing: Techniques and Applications 2018 (DICTA 2018). 	
\end{IEEEbiography}

\begin{IEEEbiography}[{\includegraphics[width=0.95 \textwidth]{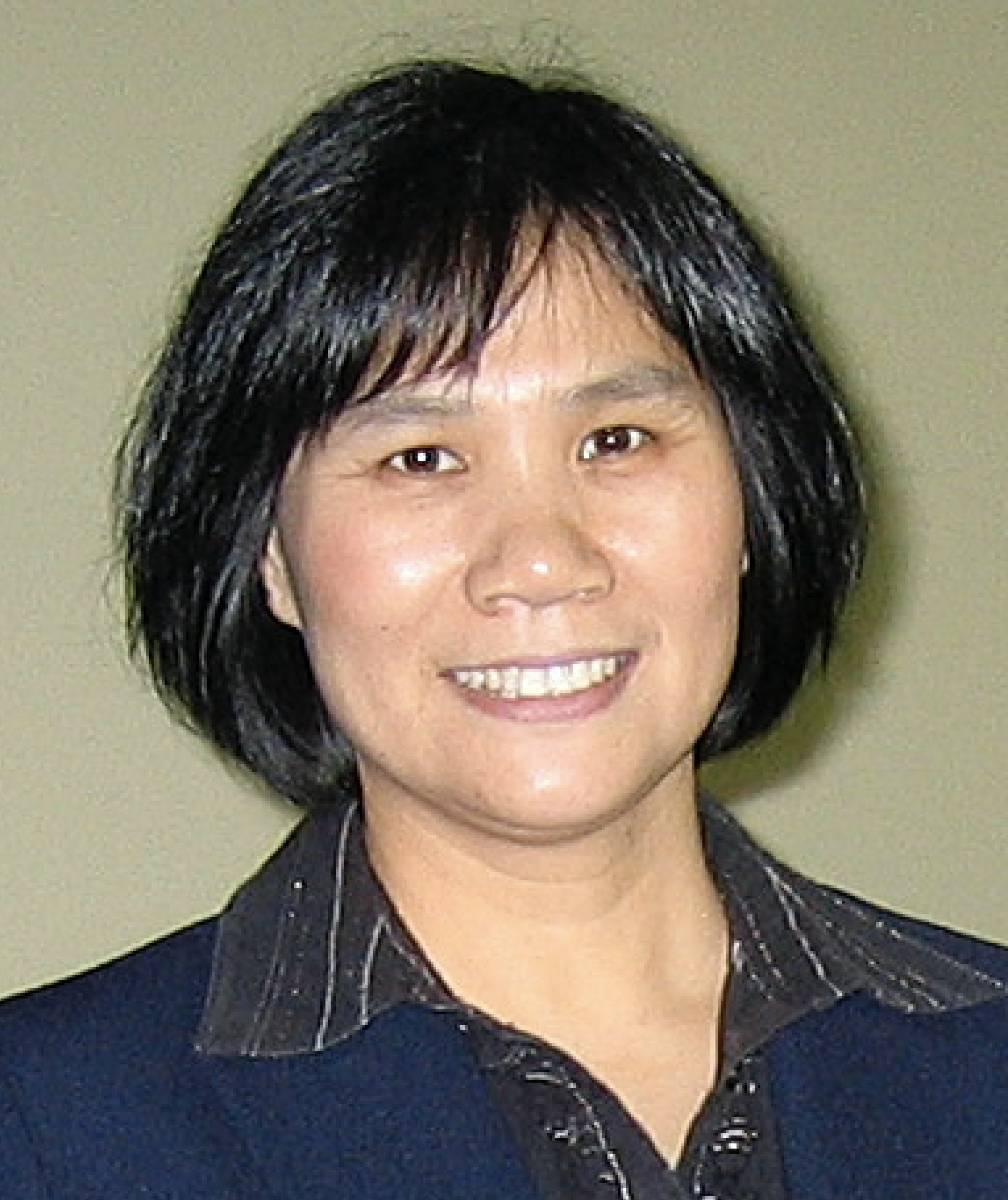}}]{Xiuping Jia}
	(M’93 \textendash SM’03) received the B.Eng. degree from the Beijing University of Posts and Telecommunications, Beijing, China, in 1982 and the Ph.D. degree in electrical engineering from The University of New South Wales, Australia, in 1996. Since 1988, she has been with the School of Engineering and Information Technology, The University of New South Wales at Canberra, Australia, where she is currently an Associate Professor. 
	Her research interests include remote sensing, image processing and spatial data analysis. Dr. Jia has authored or coauthored more than 200 referred papers, including over 100 journal papers with h-index of 34 and i10 of 102. 
	She has co-authored of the remote sensing textbook titled Remote Sensing Digital Image Analysis [Springer-Verlag, 3rd (1999) and 4th eds. (2006)]. 
	She is a Subject Editor for the Journal of Soils and Sediments and an Associate Editor of the IEEE TRANSACTIONS ON GEOSCIENCE AND REMOTE SENSING.
\end{IEEEbiography}

\begin{IEEEbiography}[{\includegraphics[width=0.95 \textwidth]{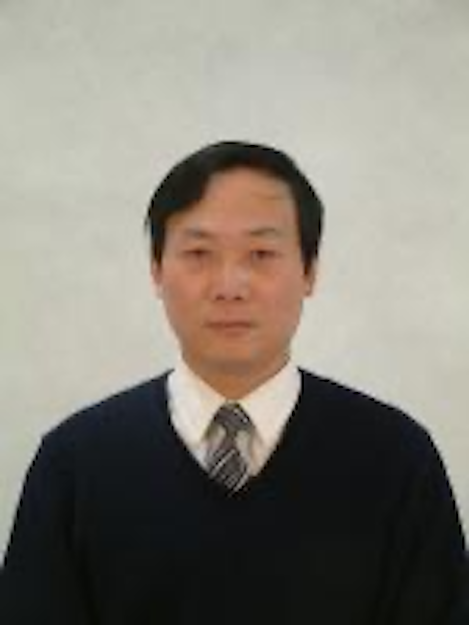}}]{Jiankun Hu}
	receive the Ph.D. degree in control engineering from the Harbin Institute of Technology, China, in 1993, and the master’s degree in computer science and software engineering from Monash University, Australia, in 2000. 
	He was a Research Fellow with Delft University, The Netherlands, from 1997 to 1998, and The University of Melbourne, Australia, from 1998 to 1999. He is a full professor of Cyber Security at the School of Engineering and Information Technology, the University of New South Wales at Canberra, Australia. 
	His main research interest is in the field of cyber security, including biometrics security, where he has published many papers in high-quality conferences and journals including the IEEE TRANSACTIONS ON PATTERN ANALYSIS AND MACHINE INTELLIGENCE. 
	He has served on the editorial boards of up to seven international journals and served as a Security Symposium Chair of the IEEE Flagship Conferences of IEEE ICC and IEEE GLOBECOM. 
	He has obtained nine Australian Research Council (ARC) Grants. He served at the prestigious Panel of Mathematics, Information and Computing Sciences, ARC ERA (The Excellence in Research for Australia) Evaluation Committee 2012.
\end{IEEEbiography}

\begin{IEEEbiography}[{\includegraphics[width=0.95 \textwidth]{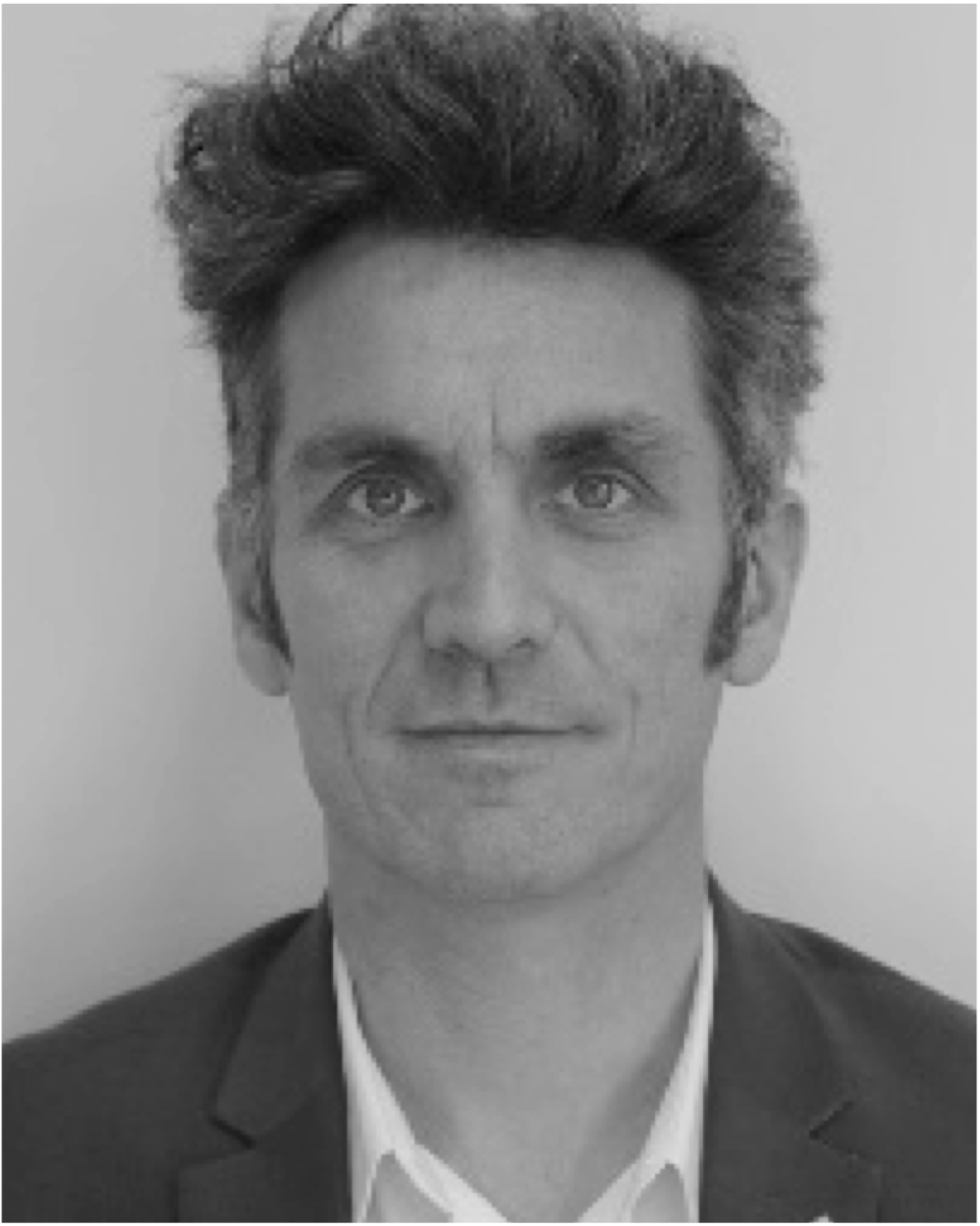}}]{Jocelyn Chanussot}
	(M’04 \textendash SM’04 \textendash F’12) received the M.Sc. degree in electrical engineering from the Grenoble Institute of Technology (Grenoble INP), Grenoble, France, in 1995, and the Ph.D. degree from the Universit de Savoie, Annecy, France, in 1998. In 1999, he was with the Geography Imagery Perception Laboratory for the Delegation Generale de l’ Armement (French National Defense Department). Since 1999, he has been with Grenoble INP, where he is currently a Professor of signal and image processing. He has been a Visiting Scholar
	with Stanford University, Stanford, CA, USA; KTH, Stockholm, Sweden; and NUS, Singapore. Since 2013, he has been an Adjunct Professor with the University of Iceland, Reykjavik, Iceland. From 2015 to 2017, he was a Visiting Professor with the University of California at Los Angeles, Los Angeles, CA, USA. He is conducting his research at GIPSA-Lab. His research interests include image analysis, multicomponent image processing, nonlinear filtering, and data fusion in remote sensing.

	Dr. Chanussot was a member of the IEEE Geoscience and Remote Sensing Society AdCom from 2009 to 2010, in charge of membership development and Machine Learning for Signal Processing Technical Committee of the IEEE Signal Processing Society from 2006 to 2008. He is a member of the Institut Universitaire de France from 2012 to 2017. He was the General Chair of the first IEEE GRSS Workshop on Hyperspectral Image and Signal Processing, Evolution in Remote sensing. He was the Chair from 2009 to 2011 and the Co-Chair of the GRS Data Fusion Technical Committee from 2005 to 2008. He was the Program Chair of the IEEE International Workshop on Machine Learning for Signal Processing in 2009. He is the Founding President of IEEE Geoscience and Remote Sensing French Chapter from 2007 to 2010 which received the 2010 IEEE GRSS Chapter Excellence Award. He was a co-recipient of the NORSIG 2006 Best Student Paper Award, the IEEE GRSS 2011 and 2015 Symposium Best Paper Award, the IEEE GRSS 2012 Transactions Prize Paper Award, and the IEEE GRSS 2013 Highest Impact Paper Award. He was an Associate Editor for IEEE GEOSCIENCE AND REMOTE SENSING LETTERS from 2005 to 2007 and Pattern Recognition from 2006 to 2008. He was the Editor-in-Chief of the IEEE JOURNAL OF SELECTED TOPICS IN APPLIED EARTH OBSERVATIONS AND REMOTE SENSING from 2011 to 2015. Since 2007, he has been an Associate Editor for the IEEE TRANSACTIONS ON GEOSCIENCE AND REMOTE SENSING, and since 2018, he has also been an Associate Editor for the IEEE TRANSACTIONS ON IMAGE PROCESSING. He was the Guest Editor for the PROCEEDINGS OF THE IEEE in 2013 and IEEE SIGNAL PROCESSING MAGAZINE in 2014.
\end{IEEEbiography}

\end{document}